\documentclass[journal]{IEEEtran}
\usepackage{cite}
\usepackage{color,xcolor}
\usepackage{amsmath,amssymb,amsfonts}
\usepackage{hyperref}
\usepackage{textcomp}
\usepackage{graphics}
\usepackage{graphicx}
\usepackage{subfigure}
\usepackage{setspace}
\usepackage{booktabs}
\usepackage[inkscapelatex=false]{svg}

\newtheorem{lemma}{Lemma}
\newtheorem{proof}{Proof}
\begin{document}
\title{Statistical Modeling of Soft Error Influence on Neural Networks}

\author{Haitong~Huang,
        Xinghua~Xue,
        Cheng~Liu,
        Ying~Wang,
        Tao~Luo,
        Long~Cheng, \\
        Huawei~Li,~\IEEEmembership{Senior Member,~IEEE,} 
        and
        Xiaowei~Li,~\IEEEmembership{Senior Member,~IEEE} 
        
\thanks{The corresponding author is Cheng Liu.}
\thanks{Haitong Huang, Xinghua Xue, Cheng Liu, Ying Wang, and Xiaowei Li are with both State Kep Lab of Processors (SKLP), Institute of Computing Technology (ICT), Chinese Academy of Sciences (CAS), Beijing 100190, China and Department of Computer Science, University of Chinese Academy of Sciences, Beijing 100190.(e-mail:liucheng@ict.ac.cn)}
\thanks{Huawei Li is with both State Kep Lab of Processors(SKLP), Institute of Computing Technology (ICT), Chinese Academy of Sciences (CAS), Beijing 100190, China and Peng Cheng Laboratory, Shenzhen, 518055, China.}
\thanks{Tao Luo is with Institute of High Performance Computing, A*STAR, 138632, Singapore.}
\thanks{Long Cheng is with North China Electric Power University, Beijing 102206, China.}
}

\markboth{IEEE Transactions on Computer-aided Design of Integrated Circuits and Systems, ~Vol.~xx, No.~xx, xxx~2022}%
{Haitong Huang \MakeLowercase{\textit{et al.}}: Bare Demo of IEEEtran.cls for IEEE Journals}

\maketitle

\begin{abstract}
Soft errors in large VLSI circuits pose dramatic influence on computing- and memory-intensive neural network (NN) processing. Understanding the influence of soft errors on NNs is critical to protect against soft errors for reliable NN processing. Prior work mainly rely on fault simulation to analyze the influence of soft errors on NN processing. They are accurate but usually specific to limited configurations of errors and NN models due to the prohibitively slow simulation speed especially for large NN models and datasets. With the observation that the influence of soft errors propagates across a large number of neurons and accumulates as well, we propose to characterize the soft error induced data disturbance on each neuron with normal distribution model according to central limit theorem and develop a series of statistical models to analyze the behavior of NN models under soft errors in general. The statistical models reveal not only the correlation between soft errors and NN model accuracy, but also how NN parameters such as quantization and architecture affect the reliability of NNs. The proposed models are compared with fault simulation and verified comprehensively. In addition, we observe that the statistical models that characterize the soft error influence can also be utilized to predict fault simulation results in many cases and we explore the use of the proposed statistical models to accelerate fault simulations of NNs. According to our experiments, the accelerated fault simulation shows almost two orders of magnitude speedup with negligible simulation accuracy loss over the baseline fault simulations.
\end{abstract}


% \fi
\IEEEpeerreviewmaketitle

\section{Introduction} \label{sec:intro}
Recent years have witnessed the widespread adoption of neural networks in various applications \cite{pouyanfar2018survey}. Many of the applications such as autonomous driving, medical diagnosis, and robot-assisted surgery are safety-critical as failures in these applications can cause threats to human life and dramatic property loss \cite{muhammad2020deep} \cite{hashimoto2018artificial} \cite{o2019legal}. 
The reliability of neural network accelerators that are increasingly utilized for their competitive advantages in terms of performance and energy efficiency \cite{reuther2021ai} \cite{reuther2019survey} becomes critical to these applications, and must be evaluated and verified comprehensively to ensure the application safety.  

With the continuously shrinking semiconductor feature sizes and growing transistor density, the influence of soft errors on large-scale chip designs becomes inevitable \cite{dixit2011impact} \cite{shafique2020robust}. A variety of analysis work have been conducted to investigate the influence of soft errors on neural network execution reliability from distinct angles recently \cite{Ares2018} \cite{Li2017understanding} \cite{xu2020persistent} \cite{xu2021reliability} \cite{xue2022winograd} \cite{torres2017fault} \cite{humbatova2020taxonomy} \cite{mittal2020survey} \cite{dos2019reliability} \cite{FIdelity} \cite{liu2022special}. For instance, Brandon Reagen et al. \cite{Ares2018} investigated the relationship between fault error rate and model accuracy from the perspective of models, layers, and structures. Guanpeng Li et al. \cite{Li2017understanding} experimentally evaluated the resilience characteristics of deep neural network systems (i.e., neural network models running on customized accelerators) and particularly studied the influence of data types, values, data reuses, and types of layers on neural network resilience under soft errors, which further inspires two efficient protection techniques against soft errors. Yi He et al. \cite{FIdelity} had the major neural network accelerator architectural parameters considered to obtain more accurate fault analysis. Dawen Xu et al. \cite{xu2020persistent} \cite{xu2021reliability} explored the influence of persistent faults on FPGA-based neural network accelerators with hardware emulation. 

Despite the efforts, they mainly rely on a large number of fault simulation on either software or FPGAs with limited fault injection configurations. The simulation based analysis is relatively accurate on specific neural network models and fault configurations, but the simulation remains rather limited compared to the entire large design space. Hence, there is still a lack of generality using the simulation based fault analysis. In fact, some of the simulation results may lead to contradictory conclusions. For instance, the experiment results in Ares \cite{Ares2018} demonstrate that the model accuracy of typical neural networks drops sharply when the bit error rate reaches $1 \times 10 ^{-7}$ while the experiments in \cite{SNR2021} reveal that the model accuracy starts to drop when the bit error rate is larger than $1 \times 10^{-5}$. In fact, both experiments are correct and the difference is mainly caused by the distinct quantization setups. Although this can be fixed with more comprehensive fault simulations, the total number of fault simulations can increase dramatically given more analysis factors, which will lead to rather expensive simulation overhead accordingly. Hence, more general analysis approaches are demanded to gain sufficient understanding of the influence of soft errors on neural networks. 

Moreover, the simulation based fault analysis can be extremely time-consuming under real-world applications with large neural networks and datasets, which hinders its use in practice. Take neural network vulnerability analysis that locates the most fragile part of an neural network to facilitate selectively protection against soft errors as an example. Suppose we need to select the most fragile $k$ layers of a neural network with $N$ layers. A straightforward simulation based approach needs to conduct $C_{N}^{k}$ experiments on the target test dataset. When $N=153, K=5$, $C_{N}^{k} = 654045930$. Assume the neural network fault simulation speed is 50 frame per second (fps) and 1000 samples are required for accuracy estimation. The evaluation of a single configuration takes 20s and the full evaluation takes around 420 years, which generally can not be afforded. As a result, many existing solutions \cite{xue2022winograd} can only adopt heuristic algorithms to address this problem approximately.

In order to achieve more efficient fault analysis of soft errors on neural network execution, we develop a statistical model to gain insight of the influence of soft errors on neural network models with much less experiments or even no experiments. The basic idea is to view the soft errors randomly distributed across the neural network processing via statistical analysis and investigate the influence of soft errors on the neural network model accuracy. Specifically, we utilize a normal distribution model to characterize the distribution of the neurons in neural networks according to central limit theorem and analyze the computing error distribution induced by the random soft errors first. On top of the models, we further investigate the influence of neural network depth, quantization, classification complexity on the resilience of neural networks under soft errors. At the same time, we verify the proposed modeling and analysis with fault simulation. Finally, we further leverage the statistical models to accelerate the time-consuming fault simulation by performing fault analysis with intermediate data rather than model accuracy directly.  

The contributions of this work can be summarized as follows.
\begin{itemize}
    \item We propose a series of statistical models to characterize the influence of soft errors on neural network processing for the first time. The models enable relatively general analysis of neural network model resilience under soft errors.
    
    \item We leverage the statistical models to investigate how the major neural network parameters such as quantization, number of layers, and number of classification types affect the neural network resilience, which can be utilized to guide the fault-tolerant neural network design. 
    
    \item With the proposed statistical models, we can also accelerate conventional fault simulation of neural network processing under soft errors by almost two orders of magnitude through simplifying the fault injection and replacing model accuracy analysis with more cost-effective intermediate parameter analysis.
    
    \item We validate the proposed model based soft error influence analysis of neural networks and demonstrate significant fault simulation acceleration with comprehensive experiments.
\end{itemize}
The rest of this paper is organized as follows. Section 2 briefly introduces prior fault analysis of neural network models. Section 3 illustrates the proposed statistical models for neural network reliability analysis under soft errors. Section 4 presents the use of the proposed statistical models to characterize the influence of neural network parameters on neural network resilience over soft errors. Section 5 mainly demonstrates how the proposed statistical models can be utilized to accelerate the fault simulation of neural networks under soft errors. Section 6 concludes this paper.

%了更好得进行软硬件设计，我们希望这些研究能得出比较一般性结论，例如哪些网络结构容错能力更强、哪些层容错能力更强等。已有的一些研究试图通过大量的实验得出一般性的结论。例如Ares对不同层、不同注入率、不同注错位置的每个情况下做足够次数的模拟实验。更广泛的实验需要耗费更多的时间。不幸的是，以往的一些研究得到的结论并不一致。例如Ares得出了不同层容错能力相差极大的结论，而一些研究则发现差距不会这样大；神经网络的敏感度不同，例如Ares发现1e-7的注入率就会导致分类准确率下降，而在我们的实验中注入率需要达到1e-5。这些差异并非是实验结果有误导致的，而是存在一些未被考虑的实验配置因素。例如Understanding Error Propagation in Deep Learning Neural Network (DNN) Accelerators and Applications 文中实验了不同的数据类型和量化方式，发现不同配置下的结果相差很大。继续考虑各种可能的因素将会导致实验的配置空间十分庞大，对于新的网络结构和配置则总是需要重新实验。如果缺乏统一的指导建议，那么对容错能力的评估总是需要等到网络设计完成后再进行注错实验。这样我们也将难以找到正确的改进方向，只能凭经验增加容错设计再进行模拟实验。盲目的结构设计和低效的模拟实验将会使整个过程十分耗时。

%由于神经网络的复杂性，硬件错误对神经网络的影响难以用传统的方式分析。目前神经网络对硬件错误容忍能力的研究通常依赖错误注入实验。例如，Ares: A framework for quantifying the resilience of deep neural networks 进行了详细的错误注入实验，得出不同模型、不同神经网络层之间的容错能力差异。然而，通过对比总结前人的工作，我们发现研究现状仍有不足之处：人们希望可靠的系统是可解释的。目前已有很多神经网络的可解释性研究，旨在理解神经网络如何工作。而对于容错硬件设计，我们还需要知道硬件错误怎样影响神经网络的准确性。基于统计的重复模拟实验虽然能够逼近真实结果，但也有无法覆盖全部情况的风险。另外，虽然人们在一些注错实验中已经得到不少有趣的现象或规律，但是这些规律能否推广到其他网络中仍是未知数。如果能对错误产生和传递的机制有一定的了解，将可以指导我们设计出更好的容错网络结构。

%在以往，该问题被视为是困难的，从而对容错能力的分析只能通过实验完成。然而，我们的分析和实验显示，虽然神经网络是高度复杂的线性系统，单个神经元上的错误的影响难以估计。但是在总体层面上，错误注入的实验结果是有规律可循的，具有明显的统计意义。通过对软错误在神经网络中的传递机制的理解，可以从理论层面上得到对上述几个重要问题的回答。然后，我们的实验很大程度上验证了该分析方法的正确性。同时该方法对以往实验中一些原因不明的现象做出合理的解释。

\section{Related Work}
Fault simulation is key to understand the influence of hardware faults on the neural network processing and is the basis for fault-tolerant neural network model and accelerator designs \cite{liu2022special} \cite{FIdelity} \cite{PytorchFI} \cite{TensorFI} \cite{SASSIFI} in various application scenarios. For instance, fault simulations in \cite{xue2022winograd} \cite{xu2021r2f} are utilized to investigate the vulnerability of neural networks and accelerators, which enables selectively hardware protection against various hardware faults with minimum overhead. Fault simulations in \cite{Ares2018} \cite{pandey2019greentpu} \cite{reagen2016minerva} are applied to investigate the design trade-offs between model accuracy loss and computing errors, which can be leveraged for energy-efficient neural network accelerator design through approximate computing and voltage scaling. Hence, a variety of fault simulation work have been developed in the past few years \cite{PytorchFI} \cite{TensorFI} \cite{xu2021r2f} \cite{gao2022soft} \cite{li2020ftt} \cite{ning2021ftt} \cite{liu2021hyca} \cite{zhang2018analyzing} \cite{chen2019BinFI} \cite{Zhen2021MindFI}. They can generally be divided into two categories depending on the fault simulation abstraction layers. 

First, neuron-wise fault simulation that injects faults to neurons or weights are mostly widely adopted in prior works \cite{liu2021hyca} \cite{PytorchFI} \cite{TensorFI} \cite{xu2021r2f} \cite{gao2022soft} \cite{li2020ftt} \cite{ning2021ftt} \cite{zhang2018analyzing} \cite{reagen2016minerva} \cite{Ares2018}  \cite{Li2017understanding} and have been verified according to \cite{Ares2018}. Although faults are originated from the underlying computing engines, these simulation frameworks typically adopt abstract bit-flip or stuck-at faults and include little hardware architecture details. To further improve the fault analysis precision, Xinghua Xue et al.\cite{xue2022winograd} developed an operation-level fault analysis framework such that hardware faults are injected to basic operations such as multiplication and accumulation, which is utilized to explore the influence of winograd convolution on resilience of neural network processing. Yi He et al. \cite{FIdelity} had neural network accelerator architectural parameters combined with high-level simulation of neural network processing with transient faults to achieve both high-fidelity and high-speed resilience study of general neural network accelerators. In summary, the above fault simulation work are mostly built on existing deep learning frameworks such as PyTorch and TensorFlow and are flexible for various fault simulations while the parallel processing capability can be negatively affected by the low-level fault injection substantially. 

The other category is circuit-layer fault simulation that typically conducts fault simulation on circuit designs at either gate level or RTL level. It is already well-supported by commercial EDA tools like TetraMAX and can achieve high simulation precision, but it can be extremely slow for neural network accelerators that include a larger number of transistors. An alternative approach is fault emulation that conducts fault simulation on FPGAs \cite{libano2018selective} \cite{gambardella2019efficient} \cite{xu2020persistent} \cite{xu2021reliability}. Similarly, NVIDIA SASSIFI\cite{SASSIFI} developed a fault injection mechanism for GPU and can be utilized for rapid fault analysis of neural networks on GPUs. Basically, these fault simulation frameworks greatly improve the fault simulation speed but rely on specific hardware prototypes and architectures which are usually difficult to scale and modify. 

Despite the efforts, simulation-based fault analysis is mainly applicable to specific setups in terms of neural network models, target hardware architectures, and fault configurations. A comprehensive fault analysis requires a huge number of fault simulations as discussed in Section \ref{sec:intro} which is prohibitively expensive. As a result, the fault analysis generality is usually limited. In fact, because of the limited fault simulation setups, some of the simulation based fault analysis even produces inconsistent results. For instance, Brandon Reagen et al. \cite{Ares2018} concluded that the resilience of different layers of the neural network may vary up to 2781$\times$. Nevertheless, Subho S. Banerjee et al. \cite{Bayesian2019} revealed a different conclusion based on the Bayesis fault injection and analysis. The problem poses significant demands for more general and faster fault analysis. 
\section{Soft Error Induced Neural Network Computing Error Modeling}
\label{sec:modeling}
In this work, we mainly analyze the influence of soft errors on neural network processing with modeling to gain insight of the neural network fault tolerance and guide the fault-tolerant design of neural network models and accelerators. Soft errors induced computing errors propagate rapidly across layers of neural networks and the influence of the different soft errors is accumulated on neurons of the neural network. Basically, the influence of random soft errors are distributed and accumulated on a large number of neurons. Hence, it can be characterized with a normal distribution model according to central limit theorem. With the distribution model, we can further estimate the neural network outputs and the model accuracy loss eventually, which can be fast and general as well. 

\subsection{Model Notations} \label{sec:notations}
Neural networks can be considered as multi-layer non-linear transformation and the transformation in a layer $l$ can be formulated as Equation \ref{eq:layer-trans} where $f_l$ represents the transformation operation, $\boldsymbol{}{x}_l$ represents the input activations, $\boldsymbol{}{x}_{l+1}$ represents the output activations. Particularly, for convolution neural networks, $\boldsymbol{w}_l$ represents weights in layer $l$, $*$ denotes convolution or full connection, $\boldsymbol{b}_l$ is the bias, and $\varphi$ represents a non-linear activation function. 

\begin{equation} \label{eq:layer-trans}
\boldsymbol{x}_{l+1}=f_{l}(\boldsymbol{x}_l)=\varphi(\boldsymbol{x}_l * \boldsymbol{w}_l + \boldsymbol{b}_l)
\end{equation}

While soft errors may happen in any layer of an neural network and propagate across the neural network layers, we utilize Equation \ref{eq:yf} to characterize the relation between input activations in layer $l$ and the output activations of the layer that is $m$ layers behind to facilitate the fault analysis. Note that $\boldsymbol{x}_l$ denotes input activations in layer $l$ and $F_l^{l+m}(\boldsymbol{x}_l)$ denotes output activations of layer $l+m$.

\begin{equation} \label{eq:yf}
F_l^{l+m}(\boldsymbol{x}_l)=f_{l+m}(f_{l+m-1}(\cdots f_{l+1}(f_l(\boldsymbol{x}_l))\cdots))
\end{equation}

%一些传统的网络鲁棒性研究，例如对抗学习，关注的是输入层$\boldsymbol{x}_1$受到扰动对$F(\boldsymbol{x}_1)$的影响。容错分析面对的问题更加复杂，因为硬件错误可能导致任意层的计算数值改变，包括激活值$\boldsymbol{x}_l$，权重$\boldsymbol{w}_l$，偏置$\boldsymbol{b}_l$。本文的分析主要关注网络内部卷积层和全连接层的activation和weight的错误，因为它们是网络计算中最频繁出现部分，更容易受到soft error的影响。例如，ResNet34的weight和activation的总数在$10^7$数量级，MAC计算的总次数达到$10^9$数量级。

Soft errors propagate along with layers of the neural network and can cause input variation on all the following layers. Suppose bit flip errors occur in weights or output activations at the $(l-1)$th layer. The induced variation at the $l$th layer is denoted as $\boldsymbol{\delta}_l$ and the variation at the $(l+a)$th layer is denoted as $\boldsymbol{\Delta}_l^{l+a}$. These variation can be calculated with \autoref{eq:feature-deviation}. For the variation of the overall neural network outputs induced by soft errors in layer $l$, we denote it as $\boldsymbol{\Delta}_l^{N}$ where $N$ refers to the total number of layers in the neural network and the notation is simplified as $\boldsymbol{\Delta}_{l}$ in the rest of this paper. 

\begin{equation}
\begin{split}
\boldsymbol{\Delta}_l^{l+1}&=\boldsymbol{x'}_{l+1}-\boldsymbol{x}_{l+1}=f_l(\boldsymbol{x}_l+\boldsymbol{\delta}_{x,l})-f_l(\boldsymbol{x}_l)\\
\boldsymbol{\Delta}_l^{l+2}&=\boldsymbol{x'}_{l+2}-\boldsymbol{x}_{l+2}=F_l^{l+1}(\boldsymbol{x}_l+\boldsymbol{\delta}_{x,l})-F_l^{l+1}(\boldsymbol{x}_l)\\
\vdots\\
\boldsymbol{\Delta}_l^{l+m}&=\boldsymbol{x'_{l+m}}-\boldsymbol{x_{l+m}}=F_l^{l+m}(\boldsymbol{x}_l+\boldsymbol{\delta}_{x,l})-F_l^{m}(\boldsymbol{x}_l)
\end{split}
\label{eq:feature-deviation}
\end{equation}

% \begin{equation}
% \begin{split}
% \boldsymbol{\delta}_{l}&\xrightarrow{\mathrm{simulate}} accuracy\\
% \boldsymbol{\delta}_{l}\rightarrow\boldsymbol{\Delta}_{l+1}\rightarrow&\boldsymbol{\Delta}_{l+2}\rightarrow\cdots\rightarrow\boldsymbol{\Delta}_y\rightarrow\ accuracy
% \end{split}
% \end{equation}

To quantize the soft error induced computing variation of the neural network, we utilized $RMSE_l=\|\boldsymbol{\Delta}_l/n\|_2=\sqrt{\mathrm{var}(\boldsymbol{\Delta}_l)}$ as a metric initially where $n$ is the vector length of the neural network output. However, RMSE is sensitive to the data range of activations that may vary over different layers of the same neural network, so we have RMSE further normalized and utilize RMSE Ratio(RRMSE) $RRMSE_l={RMSE_l}/{\sqrt{\mathrm{var}(\boldsymbol{x}_l)}}$ instead. The metric is more convenient to calculate compared to the model accuracy that relies on statistical results of a large number of samplings. The correlation between RRMSE and model accuracy will be illustrated in the rest of this section.

\iffalse
\begin{equation}
\begin{split}
RMSE_l=\sqrt{\frac{(\Delta_l^{(0)})^2+(\Delta_l^{(1)}^2)+\cdots+(\Delta_l^{(n)}}{n})^2}=\|\boldsymbol{\Delta}_l/n\|_2\\
RRMSE_l={RMSE_l}/{\sqrt{\mathrm{var}(\boldsymbol{x}_l)}}
\end{split}
\end{equation}
\fi

%我们发现，RMSE不仅能够揭示错误在网络中的传递，在注错分析中增加RMSE这样的中间指标有诸多好处。例如，更容易通过计算估计，可用于计算多个位点注错的组合误差，在实验中它比分类准确率具有更快的收敛速度。

% 下文将分析过程分为多个阶段，逐步形成完整的模型。

% \begin{enumerate}
% \item 在节\ref{sec:para_distribution}中，我们对神经网络数值的分布$\boldsymbol{x}_l, \boldsymbol{w}_l$做出分析和假设。
% \item 在节\ref{sec:value_flip_model}中，我们分析不同量化配置下单点位翻转造成的改变量$\delta_{x, l}^{(i)}$。
% \item 在节\ref{sec:error_shape}中，分析传递误差$\boldsymbol{\Delta}_l$的分布特征，然后在\ref{sec:error_energy}中分析误差在传递中的变化$\boldsymbol{\Delta}_l\rightarrow\boldsymbol{\Delta}_{l+1}$。
% \item 在节\ref{sec:model_RMSE_acc}中，分析输出误差到分类准确率的影响$\boldsymbol{\Delta}_y \rightarrow accurancy$。
% \end{enumerate}
% 这样，神经网络某一层上的单点误差到最终分类准确率影响$\delta_{x, l}^{(i)}\rightarrow accurancy$可以形成一个完整的传递链条。在之后的章节中，我们将进一步分析多点和不同层位置错误注入的组合误差。

\subsection{Assumptions and Lemmas} \label{sec:assumption}
Recent work \cite{Weight2018}\cite{Weight2020} already demonstrated that the distribution of weights in neural networks fits well with t-Location scale distribution which is essentially a long-tail normal distribution. While output activations are generally accumulation of many weighted input activations. Suppose the input activations are random variables. Then, the output activations will be close to a normal distribution according to central limit theorem. Particularly, activations are usually close to zero to make full use of the non-linear activation function. Similarly, activation errors are also accumulation of multiple random errors propagated from neurons in upstream layers and belongs to a normal distribution. In summary, weights, activations of the neural network, and activation errors can all be approximated to normal distribution centered at zero and they can be formulated as follows.

\begin{equation}
\begin{split}
&w_l \sim N(0, \mathrm{var}(w_l))\\
&x_l \sim N(0, \mathrm{var}(x_l))\\
\end{split}
\end{equation}
\begin{equation}
\Delta_l \sim N(0, \mathrm{var}(\Delta_l)) \\
\end{equation}

%\textbf{关于神经网络的权重分布}。近期有关权重集成学习的研究中，人们对神经网络的权重分布进行拟合，使用 (一种长尾正态分布)拟合结果较好。有关正则化的解释也倾向于正态分布：为了减少神经网络中训练中的过拟合，人们通常会对参数的大小加上损失惩罚，例如L2正则化。而根据贝叶斯理论，对最小二乘回归加上L2范数的约束(称为Ridge回归)，相当于对参数分布加上高斯分布的先验条件。

%根据中心极限定理，大量的随机变量相加，和的分布趋向于正态分布。神经网络的连接、卷积运算均需要对激活值做大量的做求和操作。因此在对网络缺乏进一步了解的情况下，使用正态分布作为激活值分布的估计是合理的。根据神经网络初始化和Batch Normalization的有关研究，良好的神经网络激活值应分布在0均值附近，且具有近似相等的方差。这是因为，常见激活函数在0位置附近为非线性区，想要最大化发挥神经网络的非线性拟合能力，更多的神经元应该工作在0值附近。各隐层激活值的方差大小应当近似相等。在神经网络初始化的研究中，人们认为初始化后的权重应该能使激活值变量被映射后保持相同的比例(例如xavier初始化)，否则容易导致训练中出现梯度消失或梯度爆炸。

To further verify the distribution of weights and activations in neural networks, we take LeNet on CIFAR-10 as an example.  \autoref{fig:lenet_weight_activation} shows the distributions of weights and activations on different layers of LeNet. It shows that the distribution is quite close to the fitted normal distribution model highlighted with orange color. We further have a bit error injected to a neuron in Conv2 of LeNet randomly. Then, we investigate the distribution of neuron errors in the following layers. Particularly, we take the first neuron in these layers as an example and fit the error distribution with a normal distribution model. The experiment result is shown in \autoref{fig:lenet_layer_error}. It can be observed that the neuron errors generally fit well with a normal distribution model centered at zero except that in Conv2 in which the input errors have not propagated comprehensively.

\begin{figure}[htbp]
    \centering
    \includegraphics[width=0.8\linewidth]{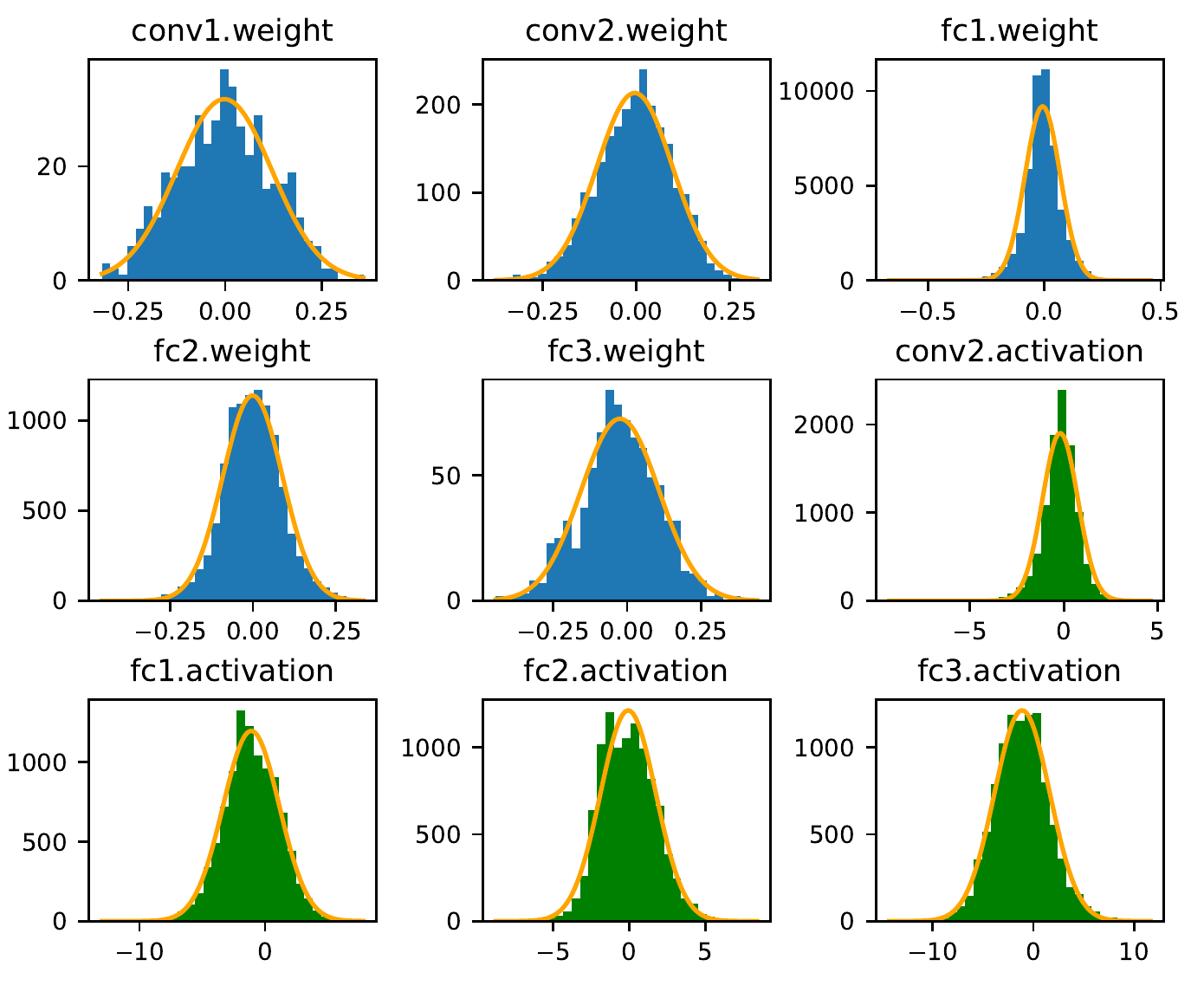}
    \caption{weights and activations in LeNet fit well with normal distribution models}
    \label{fig:lenet_weight_activation}
\end{figure}

\begin{figure}[htbp]
    \centering
    \includegraphics[width=0.7\linewidth]{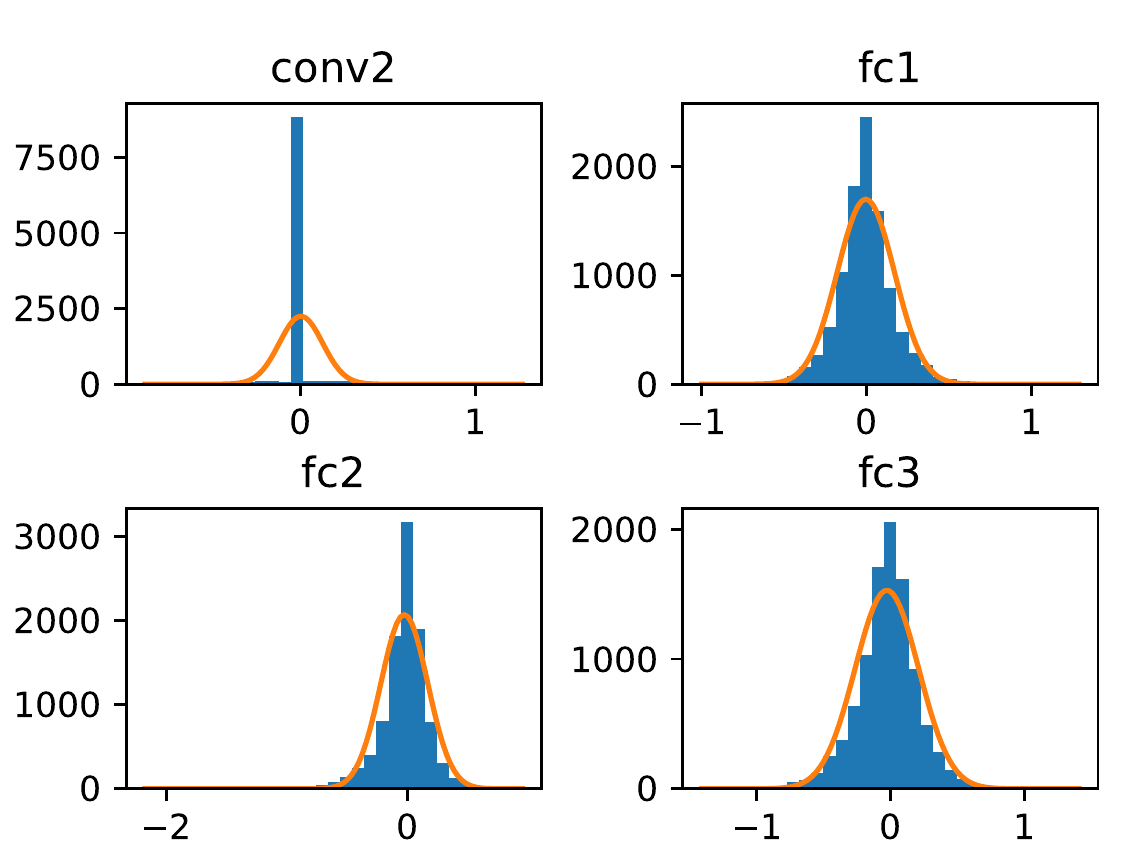}
    \caption{Error distribution of the first neuron of layers Conv2, fc1, fc2, fc3 in LeNet. Note that a single bit error is injected to an input activation of Conv2 randomly and we conduct the error injection multiple times to obtain the error distribution.}
    \label{fig:lenet_layer_error}
\end{figure}

% Error propagation in neural networks is very fast and starts to affect the majority of the following neurons after only a few layer propagation because of the closely correlated convolution and full connection operations. Figure \ref{fig:vgg11_spread} presents the error propagation of a single bit flip error from the first convolution layer in VGGNet-11. It can be seen that all the neurons are affected with 2-layer propagation xxx.  

% \begin{figure}[htbp]
%     \centering
%     \includesvg[width=1.0\linewidth]{figures/vgg_spread.svg}
%     \caption{Affected neurons in VGGNet-11 }
%     \label{fig:vgg11_spread}
% \end{figure}

In addition, we assume the distribution of the weights and activations are independently and identically to simplify the modeling in the rest of this work. With the above assumptions, we can further derive the following lemmas.
\begin{lemma}
For a layer with $n$ neurons that follow independently identical normal distribution, the variance of sum of the neurons in the layer can be approximated with $n\mathrm{var}(x)$.
\end{lemma}
\begin{lemma}
For an output activation that is an accumulation of weighted input neurons $x_{l+1}^{(j)}=\sum_{i=1}^m{x_l^{(i)}\cdot w_l^{(i)}}$, $\mathrm{var}(x_{l+1})$ can be calculated with $\mathrm{var}(x_{l+1})\approx m\mathrm{var}(x_{l})\mathrm{var}(w_l)$ when the influence of activation function is ignored according to Proof 1. Note that $m$ stands for the total number of accumulated operations for a single neuron calculation. $x_{l}^i$ and $w_l^{i}$ refer to an neuron and a weight in $(l+1)$th layer respectively.
\end{lemma}
\begin{lemma}
For an output activation error propagated from input activation errors, $\mathrm{var}(\Delta_{l+1})$ can be calculated with $\mathrm{var}(\Delta_{l+1})\approx m\mathrm{var}(\Delta_{l})\mathrm{var}(w_l)$ where $m$ represents the number of accumulation according to Proof 2.  
\end{lemma}
\begin{lemma}
With the second and the third lemmas, we can conclude that $RMSE_l/\sqrt{\mathrm{var}(x_l)} \approx  RMSE_{l+1}/\sqrt{\mathrm{var}(x_{l+1})}$. Hence, $RRMSE_l \approx RRMSE_{l+1}$ which indicates that $RRMSE_l$ keeps almost constant across different layers of an neural network.
\end{lemma}

% $\Delta_l\sim N(0,\sigma_\Delta^2)$
% For a neural network layer with $n$ activations, the accumulated standard deviation of an activation error $\|\boldsymbol{\Delta}_l\|_2=n\sigma_\Delta^2$ i.e. $RMSE_l = \sigma_\Delta$.

% Since an activation can be calculated with $x_{l+1}=\sum_{i}{x_l^{(i)} w_l^{(i)}}$

% Accordingly, $\Delta_{l+1}=\sum_{i=1}^{n}{\Delta_l^{(i)} w_l^{(i)}}$ where $\Delta\times w$ is random variable. When $n$ is sufficiently large, then the sum of the variables also follows normal distribution model and the variation can be calculated with $\mathrm{var}(\sum_{i=1}^{n}{\Delta w})\approx n \mathrm{var}(\Delta\cdot w)$.

\begin{proof}
Suppose $X$ and $Y$ are two independent random variables following normal distribution, the variation of their product can be calculated with the product of their variation according to \autoref{eq:XY-expect}. Since $x_l^{i}$ and $w_l^{i}$ can be considered as independent variables following normal distribution, we can conclude $\mathrm{var}(x_{l+1})\approx m\mathrm{var}(x_{l})\mathrm{var}(w_l)$ according to \autoref{eq:XY-expect} and Lemma 1.
\end{proof}

\begin{equation}
\label{eq:XY-expect}
\begin{split}
    \mathrm{var}(X\cdot Y) &= \mathbb{E}[X^2Y^2]-\mathbb{E}^2[XY] \\
    &=\mathbb{E}[X^2]\mathbb{E}[Y^2]-\mathbb{E}^2[X]\mathbb{E}^2[Y]\\ 
    &=\mathrm{var}(X) \cdot \mathrm{var}(Y)
\end{split}
\end{equation}

\begin{proof}
An output activation error induced by neuron errors in previous layer can be formulated with \autoref{eq:neuron-error-cal}. Basically, an neuron error is essentially the accumulation of weighted input neuron errors. As mentioned, both neuron errors and weights can be characterized with normal distribution. In this case, the variation of an neuron error can also be calculated with $\mathrm{var}(\Delta_{l+1})\approx m\mathrm{var}(\Delta_{l})\mathrm{var}(w_l)$ according to \autoref{eq:XY-expect} and Lemma 1.
\end{proof}

\begin{equation}
\label{eq:neuron-error-cal}
\begin{split}
\Delta_{l+1}^{(j)}&=((\boldsymbol{x}_l+\boldsymbol{\Delta}_l)*\boldsymbol{w}_l-\boldsymbol{x}_l*\boldsymbol{w}_l)^{(j)}\\
&=(\boldsymbol{\Delta}_l*\boldsymbol{w}_l)^{(j)}\\
&=\sum_{i=1}^m{\Delta_l^{(i)}\cdot w_l^{(i)}}
\end{split}
\end{equation}

\subsection{Bit Flip Influence Modeling \label{sec:value_flip_model}}
To understand the influence of bit flip soft errors on neural network processing, we start to investigate the influence of bit flip on a single data. Take an activation $x$ quantized with int8 as an example, the quantized activation $x_Q$ can be represented with \autoref{eq:xq} where $bound$ represents the dynamic range of activations in a layer of the neural network. 
\begin{equation}
\label{eq:xq}
    x_Q=\lfloor \frac{128\cdot x}{bound} \rfloor
\end{equation}

%软错误表现为数值位上的随机位翻转。分析随机注错对整个网络输出的影响，应该从分析随机注错对单个变量的影响$\delta_{x, l}^{(i)}$开始。在常见的量化模型中，神经网络的权重$\boldsymbol{w}$和激活值$\boldsymbol{x}$一般被线性地映射到n位整数上。我们以对称的8位量化为例，公式为：
% 这里的bound为变量的动态范围(dynamic\_range)，通常从网络的实际数值做统计得到。
% 在实际情况下的bound的选取有多种选取策略：
% \begin{itemize}
% \item 我们的实验中，每个层的激活值和权重使用不同的bound。为了计算方便，bound取x.abs().max()或x的高斯分布的n个标准差范围。
% \item 在实际应用中，bound通常取2的幂次，以方便硬件做移位运算。
% \item 在Ares的实验中，作者测试了在网络全局上使用2.12和3.13两种量化策略。2.12即2为整数12位小数，bound=4；3.13为3位整数，bound=8。
% \end{itemize}

% 如果被注错的数值在[-bound, bound]区间内均匀分布，容易计算它受到错误注入后的平均改变量。对于int8的量化，假设错误注入是8个位置上的均匀随机位翻转，则期望的改变量为：$$(bound+bound/2+\cdots+bound/64)/8\approx bound/4$$
% 其他设置下的期望改变量：
% \begin{itemize}
%     \item 对于如果把随机位翻转改为最高位翻转，则期望改变量为bound，为随机位翻转的4倍。
%     \item 如果在不改变bound的情况下把int8量化改为int16量化，则分母应该除以16，得到的结果为bound/8，也就是int8的1/2。注意这里的错误指的是在整个数值上随机选取1位。如果错误率是对每一个位来计算的，则int16的位数量是int8位数量的两倍，最终两者平均误差是一样的。
%     \item 误差的大小和bound大小是线性关系的，bound增加一倍会导致误差也增加一倍。
% \end{itemize}
% 上述的计算是按照误差绝对值计算的的。
In this section, we mainly investigate the influence of quantization $bound$ and the different fault injection methods on the resulting variations of activations. As mentioned before, we utilize RMSE as the metric to measure the activation variation caused by soft errors. Suppose the range of an activation is $[-bound, bound]$, a bit flip on most significant bit results in a change of $\pm bound$. Similarly, a bit flip on the following bit leads a change of $\pm bound/2$. In this case, the expected change of data i.e. $\sigma_\delta$ quantized with int8 and int16 given a random bit flip can be represented with \autoref{eq:int8-error-expectation} and \autoref{eq:int16-error-expectation} respectively. When we double the $bound$, $\sigma_\delta$ doubles. When the quantization data width doubles, $\sigma_\delta$ shrinks by $\sqrt{2}$. 

\begin{equation}
\label{eq:int8-error-expectation}
\sigma_{\delta_{int8}}=\sqrt{\frac{1}{8}\sum_{bit=0}^{7}{\left(\frac{bound}{2^{bit}}\right)^2}}\approx \frac{bound}{\sqrt{6}}
\end{equation}

\begin{equation}
\label{eq:int16-error-expectation}
\sigma_{\delta_{int16}}=\sqrt{\frac{1}{16}\sum_{bit=0}^{15}{\left(\frac{bound}{2^{bit}}\right)^2}}\approx \frac{bound}{\sqrt{12}}
\end{equation}

%When we conduct fault injection on only the highest bit rather than random bits of the data on int8 quantization, RMSE turns out to be $\sqrt{6} \times$ higher.

On top of the bit error influence of a single data, we scale the analysis to estimate the output data variation of a convolution operation induced by soft errors injected to either weights or activations. Suppose $ic$ refers to the number of input channel, $oc$ refers to the number of output channel, $H \times H$ stands for the size of a feature map, $K \times K$ refers to the size of a convolution kernel.

\textbf{Disturbance in weights:} Disturbance in a single weight affects output activations of an entire channel, i.e. $H^2$ activations. According to definition $RMSE = \|\boldsymbol{\Delta}/n\|_2$ and Lemma 2 in Section \ref{sec:assumption}, the average variation of weights i.e. $var(w)$ depends on the amount of input activations $m=K^2\times ic$. $var(w) =  \frac{1}{m} var(x_{l+1})/var(x_l) \approx 1/m = 1/(K^2\times ic)$. So the average computing variation of output activations can be calculated with Equation \ref{eq:estimate_weight}.

\begin{equation}\label{eq:estimate_weight}
\begin{aligned}
RMSE_w &= \|\boldsymbol{\Delta}/n\|_2\\
&= \sqrt{\frac{H^2 \times \Delta^2}{H^2 \times oc}} \\ 
&= \sqrt{\frac{H^2\sigma_\delta^2 var(w)}{H^2 \times oc}}\\
&=\frac{\sigma_\delta}{K\sqrt{ic\times oc}}
\end{aligned}
\end{equation}

\textbf{Disturbance in activations:} Suppose the average value of an activation in a layer is $\sigma_a$. Disturbance in a single activation will affect a window ($K^2$) of output activations in all the $oc$ channels, i.e. $oc\times K^2$ activations. Suppose the feature map size is $H \times H$, then the average influence of the disturbance on an output activation can be calculated with Equation \ref{equ:act-est}. 

\begin{equation}
\label{equ:act-est}
\begin{aligned}
RMSE_a &= \|\boldsymbol{\Delta}/n\|_2\\ 
&= \sqrt{\frac{oc\times K^2 \times \sigma_\delta^2 var(w) }{H^2\times oc}}\\
&=\frac{\sigma_\delta}{H\sqrt{ic}}
\end{aligned}
\end{equation}

Since the variation of activations and weights is generally constant given a specific model and data set according to the assumptions in Section \ref{sec:assumption}, RRMSE of a convolution layer $l$ that can be calculated with $RRMSE_l={RMSE_l}/{\sqrt{\mathrm{var}(\boldsymbol{x}_l)}}$ is consistent with RMSE accordingly.

\subsection{Relation between RRMSE and Classification Accuracy \label{sec:model_RMSE_acc}}
%本文前面模型中的错误估计都是以RMSE为基础的。但是在实际应用中，容错设计的错误注入实验可能更关注分类准确率等指标。例如，以往的多个注错实验，包括Ares等，关注的是分类错误率随注入率变化。下面我们来定性地分析RMSE大小如何影响分类准确率等模型指标即$\boldsymbol{\Delta}_y\rightarrow accurancy$，以及为何会形成S型的错误率曲线。
Since the model accuracy of a typical classification task is based on statistics of a number of classification tasks and it is difficult to formulate with neural network computing directly, we utilize RRMSE defined in \autoref{sec:notations} as an alternative metric to measure the influence of soft errors on model accuracy metric which is more closely related with neural network computing. 

To begin, we start with a simple binary classification task and it includes only two output neurons in the last layer of the network. Then, the classification depends on larger output neuron. Suppose the two output neurons are $y^{(0)}$ and $y^{(1)}$, and assume $y^{(0)}>y^{(1)}$ without loss of generality. When there are random errors injected to the neural network, the two output neurons follow normal distribution accordingly and they can be formulated with $Y^{(0)}\sim N(y^{(0)}, \mathrm{var}(\Delta_y)$, $Y^{(1)}\sim N(y^{(1)}, \mathrm{var}(\Delta_y))$ according to Section \ref{sec:assumption}. In this case, the probability of wrong classification is essentially that of $Y^{(0)}<Y^{(1)}$. The distribution of $Y^{(0)}<Y^{(1)}$ can be formulated with \autoref{eq:error-dist}, which is a shifted and scaled normal distribution model. Hence, the probability when $Y^{(0)}-Y^{(1)}>0$ can be calculated with \autoref{eq:gaussian-prob} where $\mathrm{normcdf}$ stands for  cumulative distribution function of normal distribution.
%首先，为了简化对问题的分析，设这里的神经网络执行一个二分类任务，即网络输出层包含两个神经元。在分类任务中的训练中，通常会在输出层之后加上softmax层，然后用交叉熵计算分类损失，此时softmax的输出可解释为网络对分类结果的置信度。在推理阶段，只需要对输出层$\boldsymbol{y}$取最大值得到分类结果。
%若错误注入导致分类错误，则${y'}^{(0)}<{y'}^{(1)}$。也就是计算的概率。$Y^{(0)}-Y^{(1)}$的分布为：

\begin{equation}
\label{eq:error-dist}
Y^{(0)}-Y^{(1)}\sim N(y^{(0)} - y^{(1)}, 2\mathrm{var}(\Delta_y))
\end{equation}

%即平移缩放后的高斯分布。因的概率为高斯分布的累积分布函数$\Phi$，函数的参数取决于原始输出的间隔和注入引起的方差：
\begin{equation} \label{eq:gaussian-prob}
\begin{aligned}
Acc &= \mathrm{normcdf}\left(\frac{y^{(0)} - y^{(1)}}{\sqrt{2\mathrm{var}(\Delta_y)}}\right)\\
&= \frac{1}{2}\mathrm{erf}\left(\frac{y^{(0)} - y^{(1)}}{2\sqrt{\mathrm{var}(\Delta_y)}}\right)+\frac{1}{2}\\
\end{aligned}
\end{equation}

While $\mathrm{var}(\mathbf{y})=\frac{1}{4}(y^{0}-y^{1})^2$ in a typical binary classification task, \autoref{eq:gaussian-prob} can be converted to Equation \ref{eq:accuracy-RMSE} according to the definition of $RMSE$ and $RRMSE$ in Section \ref{sec:notations}.
\begin{equation} \label{eq:accuracy-RMSE}
\begin{aligned}
Acc 
&= \frac{1}{2}\mathrm{erf}\left(\frac{2\sqrt{\mathrm{var}(\mathbf{y})}}{2RMSE_y}\right)+\frac{1}{2}\\
&= \frac{1}{2}\mathrm{erf}\left(\frac{1}{RRMSE_y}\right)+\frac{1}{2}
\end{aligned}
\end{equation}

\begin{figure*}
\begin{equation} \label{eq:accuracy-RMSE-multi}
\begin{aligned}
Acc(RRMSE)
&= \int_{-\infty}^{\infty}{\mathrm{normpdf}(x)\cdot\left(\mathrm{normcdf}(x+1/RRMSE)\right)^{nc-1}dx}\\
&= \int_{-\infty}^{\infty}{\frac{1}{\sqrt{2\pi}}e^{-\frac{x^2}{2}}\cdot\left(\int_{-\infty}^{t}{\frac{1}{\sqrt{2\pi}}e^{-\frac{(t+1/RRMSE)^2}{2}}dt}\right)^{nc-1} dx}
\end{aligned}
\end{equation}
%\hrulefill
\end{figure*}

%类似二分类模型，假设目标类别输出为y(0)，其他类别输出y(1...999)相等并且比y(0)小，此时网络能正确分类。发生错误的情况是对y(0...999)加上高斯分布的扰动后，y(1...999)中的任意一个比y(0)更大。设y(0)与y(1...999)的扰动大小(标准差)与原始间隔之比为RRMSE。可以根据定义计算能够正确分类的概率，它是y(0)密度函数normpdf与y(1...999)累积分布函数乘积的积分。

\begin{figure*}
\begin{equation}
\label{eq:empirical-accuracy-RMSE-multi}
\begin{aligned}
Acc(RRMSE)
&= \frac{(1+e^{-ms})(Acc_{clean}+nc^{-1})}{1+e^{s(RRMSE-m)}}+nc^{-1}
\end{aligned}
\end{equation}
\vspace{-1em}
\end{figure*}
%鉴于上式计算过于困难，而且在多样本情况下无法建模。从我们建议用一个更简单的sigmoid函数拟合作为经验公式。该函数表达式由sigmoid函数变形得到，它保证RRMSE=0时为acc_free且RRMSE->\infty时为1/nclass。m, s为用于拟合的两个形状参数，此时拟合整条曲线最少需要两次注错实验，更多数据能让曲线更精确，实验发现适当注错率下3个点就能准确拟合。

\begin{figure}[htbp]
    \centering
    %\includesvg[width=0.7\linewidth]{figures/rate_acc.svg}
    \subfigure[Single Image Analysis]{
        \includegraphics[width=1\linewidth]{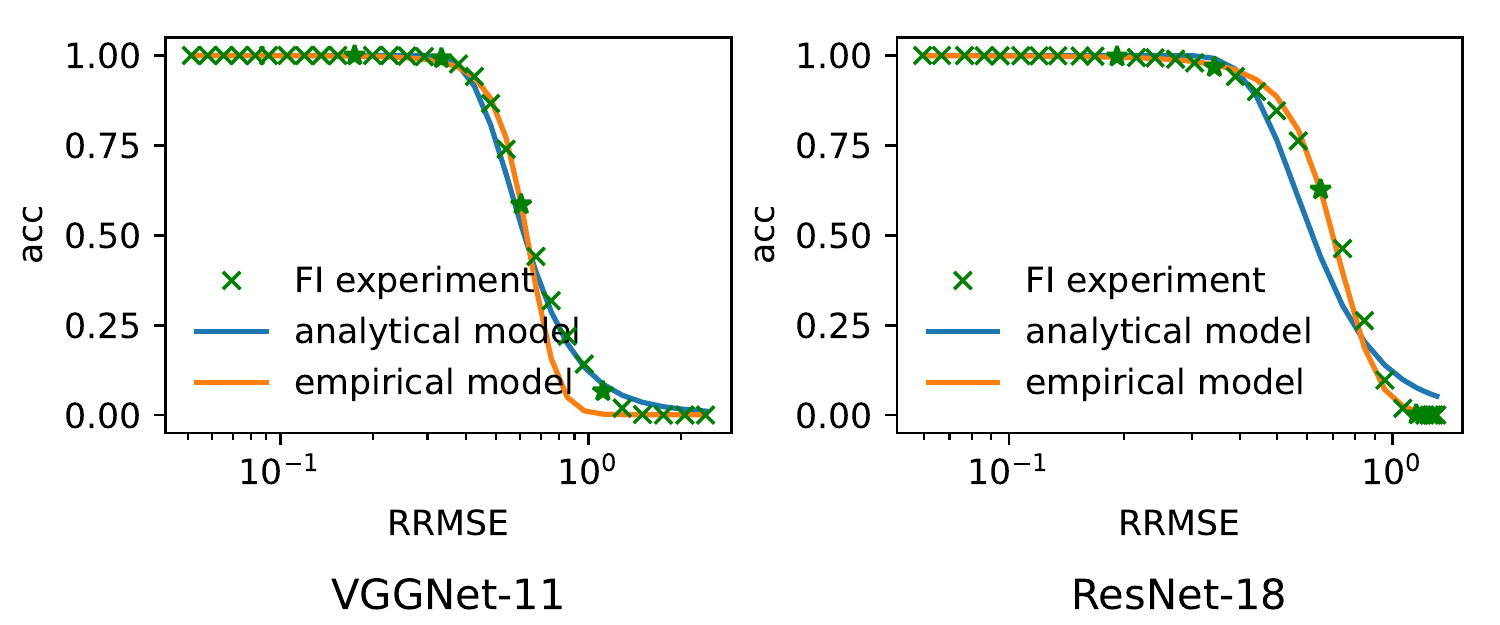}
    }
    \subfigure[Multiple Image Analysis]{
        \includegraphics[width=1\linewidth]{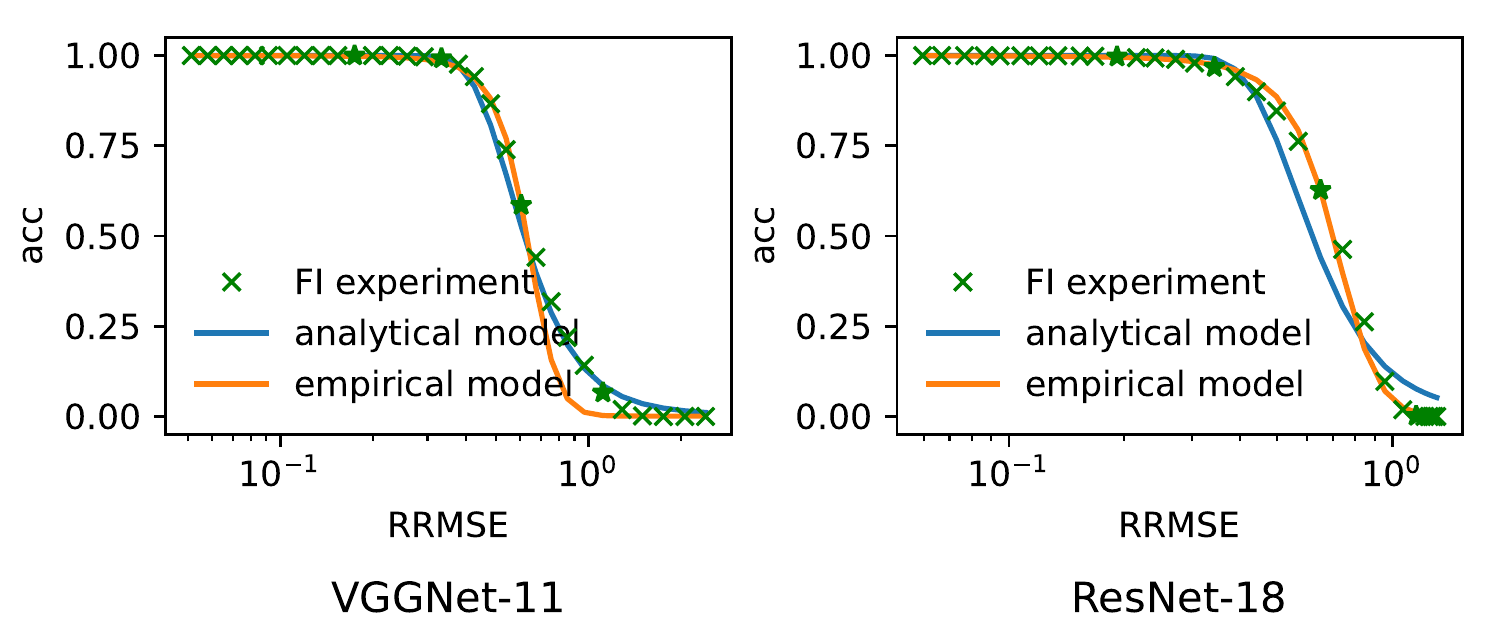}
    }
    \caption{Analytical model and empirical model comparison on both single image analysis experiment and multiple image analysis experiment.}
    \label{fig:exp_rmse_acc}
    \vspace{-1.5em}
\end{figure}

For multi-class classification models, suppose there are $nc$ classification types and the corresponding outputs of the classification model are denoted as $y(i), i=(0, 1, 2, ..., nc-1)$. Assume the expected output is $y(0)$ which is larger than any other output $y(i), i=(1, 2, ..., nc-1)$. Similar to the binary classification problem, a classification error happens when any of the outputs $y(i), i=(1,2 , ..., nc-1)$ is larger than $y(0)$ given Gaussian disturbance i.e. $RRMSE$ according to the definition of RRMSE. In this case, the model accuracy subjected to errors is the integration of the multiplication of the probability density function of $y(0)$ and the cumulative distribution function of the rest outputs $y(i), i=(1, 2, ..., nc-1)$. It can be calculated with Equation \ref{eq:accuracy-RMSE-multi}. 

Since the analytical model is usually difficult to calculate, we replace it with an empirical model as shown in Equation \ref{eq:empirical-accuracy-RMSE-multi}. It is essentially an sigmoid function variant and has two parameters i.e. $m$ and $s$ included. When there are no errors, the output of Equation \ref{eq:empirical-accuracy-RMSE-multi} is the accuracy of a clean neural network model. When there are too many errors, the output of Equation \ref{eq:empirical-accuracy-RMSE-multi} becomes $1/nc$ which represents classification accuracy of random guessing. The empirical model can be determined given very few sampling data points of RRMSE and the corresponding classification accuracy.

To verify the proposed models of correlation between RRMSE and neural network classification accuracy, we take VGGNet-11 and ResNet-18 on ImageNet as typical neural network examples to compare the analytical models and empirical models to the ground truth results. Particularly, we have two different evaluation approaches performed for the model comparison. In the first approach, we have a single true positive images from ImageNet utilized and repeat the execution for 10000 times on each error injection rate setup. It is denoted as single image analysis and it removes the influence of image variations from the analysis. In the second approach, we have 10000 different images randomly selected from ImageNet and evaluated for each error injection rate setup. It is denoted as multiple image analysis and it has the image variations incorporated in the analysis. For the soft error injection, we have 31 different bit error rate (BER) setups ranging from 1E-7 to 1E-4 conducted in the experiment. Ground truth of RRMSE and classification accuracy can be obtained from experiments directly. Analytical model can be determined given the RRMSE and $nc$ while the empirical model can be determined with fitting on 4 data points evenly selected from ground truth data. The comparison is shown in Figure \ref{fig:exp_rmse_acc}. It can be observed that the proposed analytical model is close to the ground truth data in the single image analysis setup but it has variation under multiple image analysis setup. This is expected as the proposed analytical model fails to characterize the influence of image variations. In contrast, empirical model fits much better on both single image analysis and multiple image analysis despite the lack of explainability. Nevertheless, it is clear that the model accuracy decreases monotonically with the increase of RRMSE, which allows us to characterize the influence of soft errors with RRMSE which can be obtained more conveniently compared to model accuracy.

\subsection{Error Influence Aggregation} \label{sec:estimate_directly}
In this section, we mainly explore how the influence of different errors aggregate on the same output neurons. Suppose $RRMSE$ represents RRMSE of an neural network output, it is the accumulation of $n$ independent random errors propagated from different layers. As mentioned, the influence of the random errors follows normal distribution model. The accumulation of these independent random errors can be formulated with Equation \ref{equ:combine} where  $RRMSE_{(l, i)}$ denotes the RRMSE induced by a neuron faults on layer $l$. As the number of neurons in a neural network is extremely large, it remains timing consuming and inefficient to conduct neuron-wise fault analysis. To address the problem, we have Equation \ref{equ:combine} further converted to layer-wise fault analysis where $RRMSE_{(l)}$ denotes RRMSE of neural network output caused by all the errors in layer $l$ and $L$ denotes the total number of neural network layers.
\begin{equation}\label{equ:combine}
RRMSE=\sqrt{\sum_{l=1}^L{\sum_{i=1}^{n_l}{RRMSE_{(l, i)}^2}}}=\sqrt{\sum_{l=1}^L{RRMSE_{(i)}^2}}
\end{equation}

To verify the error aggregation model, we conduct layer-wise fault injection on VGGNet-11 and Resnet-18 quantized with int8 at different error injection rate $BER\in\{0, 1e-5, 2e-5\}$. Then, we randomly select 32 different combination of layer-wise error injection configurations and compare the ground truth $RRMSE$ with that estimated with Equation \ref{equ:combine}. The comparison shown in Figure \ref{fig:layerwise_combination} reveals that the error aggregation model fits well with the simulation results and confirms the effectiveness of the model.

\begin{figure}[htbp]
    \centering
    \subfigure[VGGNet-11]{
        \includegraphics[width=0.45\linewidth]{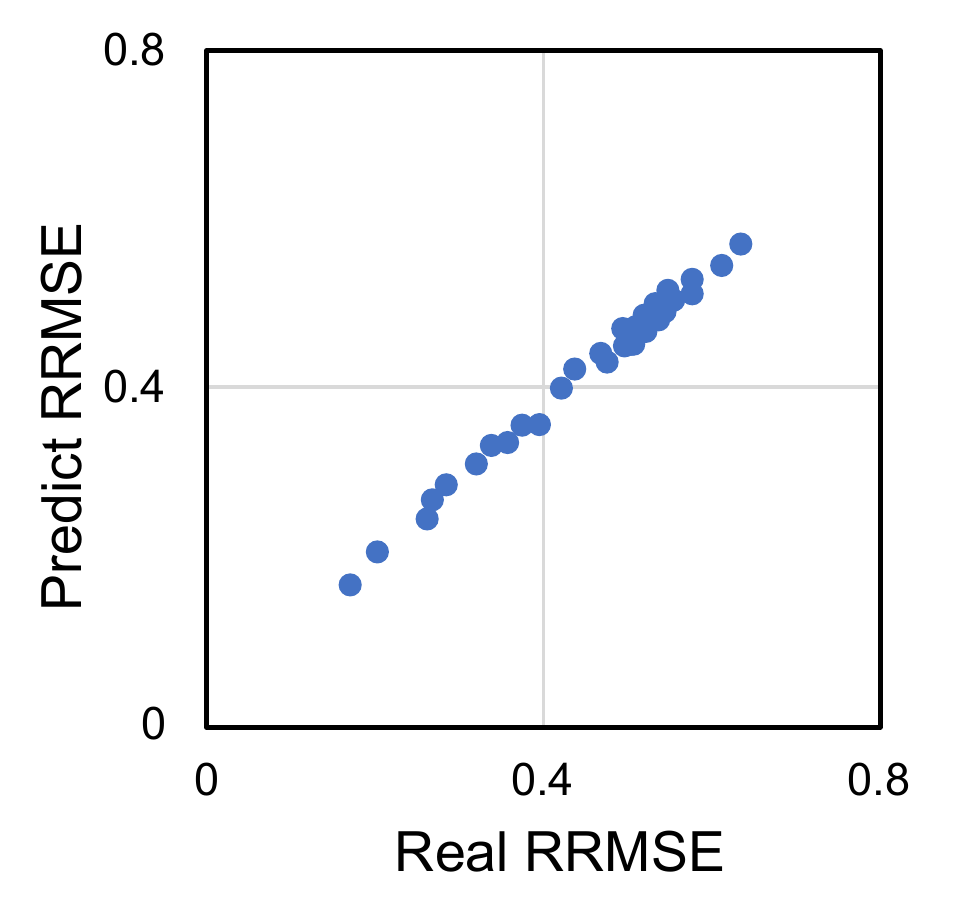}
    }
    \subfigure[ResNet-18]{
        \includegraphics[width=0.45\linewidth]{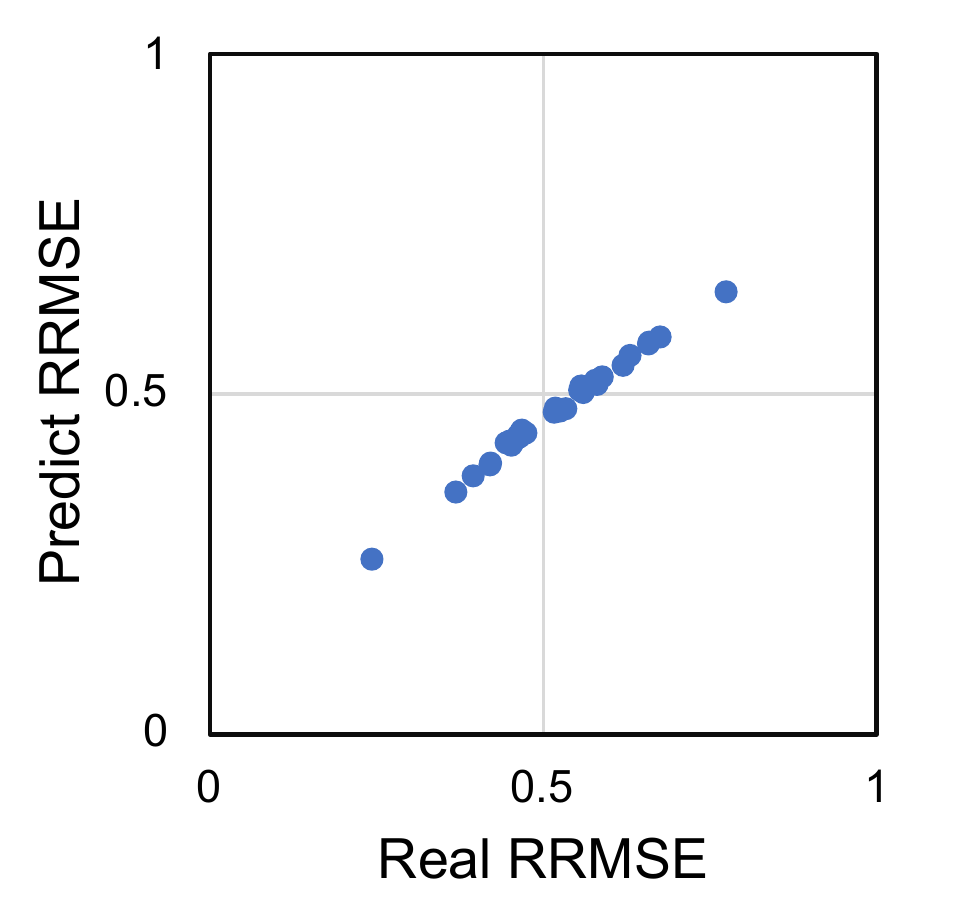}
    }
    \caption{Comparison of ground truth RRMSE and that estimated with error aggregation model.}
    \label{fig:layerwise_combination}
\end{figure}

\section{Modeling for Neural Network Resilience Analysis}
With the above modeling of soft error influence on neural network accuracy, we can further utilize the models to analyze the influence of the major neural network design parameters on the resilience of neural networks subjected to soft errors, which can provide more general analysis compared to simulation based approaches. Specifically, we investigate influence of neural network layers, quantization, and number of classification types respectively on neural network resilience and they will be illustrated in detail in this section.

\subsection{Influence of Number of Layers on NN Resilience} \label{sec:analyze_layer}
According to Lemma 4 in Section \ref{sec:assumption}, the output error metric of the $l$th layer i.e. $RRMSE_l$ is almost constant across different layers of a neural network, which means that neural network depth does not have direct influence on neural network resilience. In order to verify this, we have random bit flip errors injected to neurons in different layers of VGGNet-11 and ResNet-18 respectively. The bit error rate is set to be 1E-6 in this experiment. Then, we evaluate RRMSE of the following layers of the neural networks over the corresponding golden reference output. RRMSE of the neural network layers is shown in Figure \ref{fig:layer-prop}. It can be seen that RRMSE on different layers varies in a small range in general for each specific fault injection despite the layer locations of the error injection. Basically, it demonstrates that the influence of soft errors remains steady across the different layers and neural network depth will not affect the fault tolerance of the neural network model in general. It is true that small variations rather than constant as estimated in Lemma 4 can be found in RRMSE of different layers in Section \ref{sec:assumption}. This is mainly caused by the non-linear operations in neural networks which may filter out the computing errors and affect the error propagation. %\textcolor{blue}{Nevertheless, it reveals that the error propagation mode in the quantization model is very different from that in the floating-point model, because an Inf or Nan error can easily lead to the complete collapse of the floating-point model according to Figure XXX.}

\begin{figure}[tb] \centering 
    \subfigure[VGGNet-11] {\label{fig:vgg11_prop}
        \includegraphics[width=0.8\columnwidth]{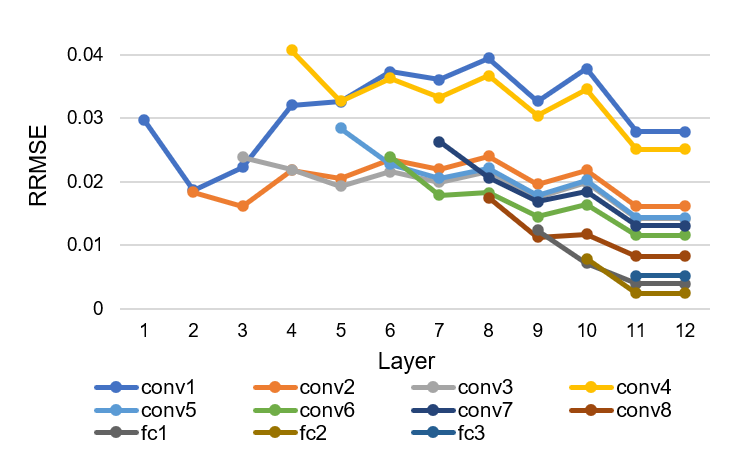} 
    }  
    \subfigure[ResNet-18] { \label{fig:ResNet-18_prop} 
        \includegraphics[width=0.9\columnwidth]{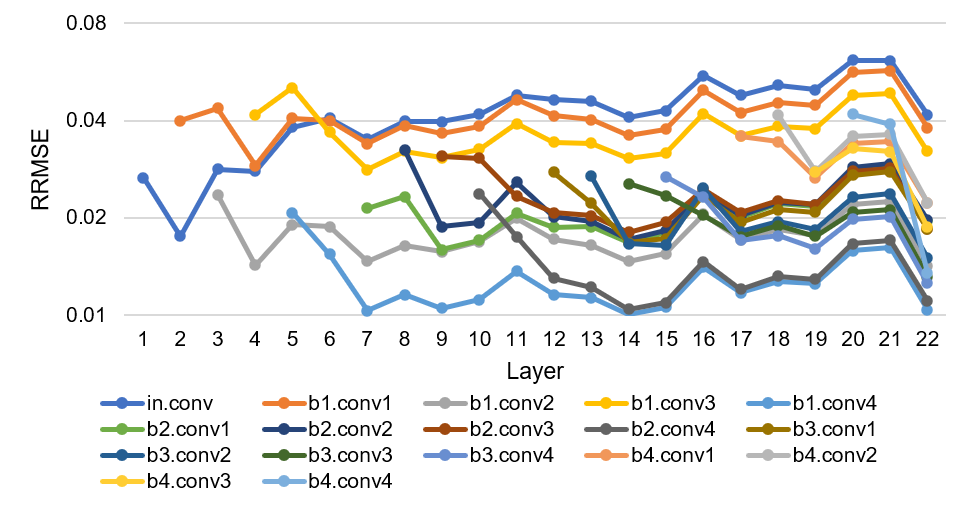} 
    } 
\caption{Error Propagation Across the Layers. Color of the lines refers to experiments with different initial error injection layers.} 
\vspace{-1em}
\label{fig:layer-prop} 
\end{figure}

%In summary, we can estimate the influence of various soft error configurations with the above analysis procedures. The estimation accuracy will be validated with experiments accordingly.
%On top of the error distribution, we also investigate the error propagation with RRMSE metric, which is closely related with the model accuracy. We conduct error injection in VGG11 and ResNet-18 per layer activation both at BER=6.25e-7. According to the experiments in \autoref{fig:vgg11_prop} \autoref{fig:ResNet-18_prop}, it can be seen that RRMSE changes slightly and have the same trend.

% \begin{itemize}
%     \item 放大。目前一些对抗攻击的研究表明，对网络的输入施加微小的扰动，足以导致模型输出出现严重错误，这种情况下网络明显放大了对抗攻击的扰动。
%     \item 缩小。在\cite{Ares2018}的几个不同网络的实验中，深层卷积的敏感度总是比浅层卷积明显更大。
%     \item 近似相等。根据\cite{Bayesian2019}等人的贝叶斯深度网络模型分析，不同深度模块对容错能力的影响并不大；\cite{Understand2017}论文实验数据中误差大小在不同层传递中也没有很大变化。我们的初步实验也倾向于这个结论。
% \end{itemize}

% 根据本文\ref{sec:error_energy}中的结论，误差在网络层中传递的平均能量$RRMSE_l$应近似保持不变。也就是说网络层神经元的可靠性与它所在深度的联系并不大。当然，我们也说过\ref{sec:error_energy}中的结论比较粗糙，该差异可能在深度网络的传递中累积。
% \textcolor{red}{
% 为此我们在多层网络中测试了$RRMSE_l$的传递，结果显示：在一个6层卷积网络中，误差能量波动不超过50\%；在ResNet34的实验中，34层的累积波动最大值不超过2.5x。这说明，如果实际观察到的可靠性有很大差距，则更有可能是受到了其他因素的影响。
% }

\subsection{Influence of Quantization on NN Resilience}\label{sec:analyze_quan}
% 在\ref{sec:model_RMSE_acc}的分析中，我们知道模型准确率与受到的扰动大小密切相关，而注错扰动大小$\delta_l^{(i)}$与数据类型密切相关。由于浮点数受到位翻转可能引起巨大的误差，我们可以推断使用浮点数据类型的网络比整数量化网络更加脆弱。\cite{Understand2017}中对比了几种不同运算数值类型的实验，发现不同数据类型的错误率相差很大，例如浮点类型在某个注入率下使精度下降了30\%，而定点类型的精度影响不到3\%。数值的各个位的敏感度也不相同，通常是高位>低位，浮点数的阶码位>>尾数位。在我们的实验中，也可以发现浮点模型的容错能力远差于定点模型。

According to the analysis in Section \ref{sec:model_RMSE_acc}, we notice that quantization bound has straightforward influence on data variation induced by a single bit flip and affects RRMSE accordingly. Meanwhile, neural network models can choose different quantization bound with little accuracy loss in practice, so it can be expected that smaller quantization bound can improve the neural network model reliability subjected to soft errors without accuracy penalty. Figure \ref{fig:bound_RMSE} presents RRMSE of VGGNet-11 and ResNet-18 subjected to soft errors with various quantization bound while the bit error rate is set to be 1E-6. It reveals that RRMSE that is closely related with the neural network model accuracy as demonstrated in Section \ref{sec:model_RMSE_acc} increases almost linearly with the quantization bound experimented with all the different layers. On the other hand, we also observe that clean neural network classification accuracy remains steady in a wide range of quantization bound setups, which indicates that it is possible to improve the neural network reliability without accuracy penalty by choosing appropriate quantization bound.

\begin{figure}[tbp]
    \centering
    \subfigure[VGGNet-11]{
    \includegraphics[width=0.45\linewidth]{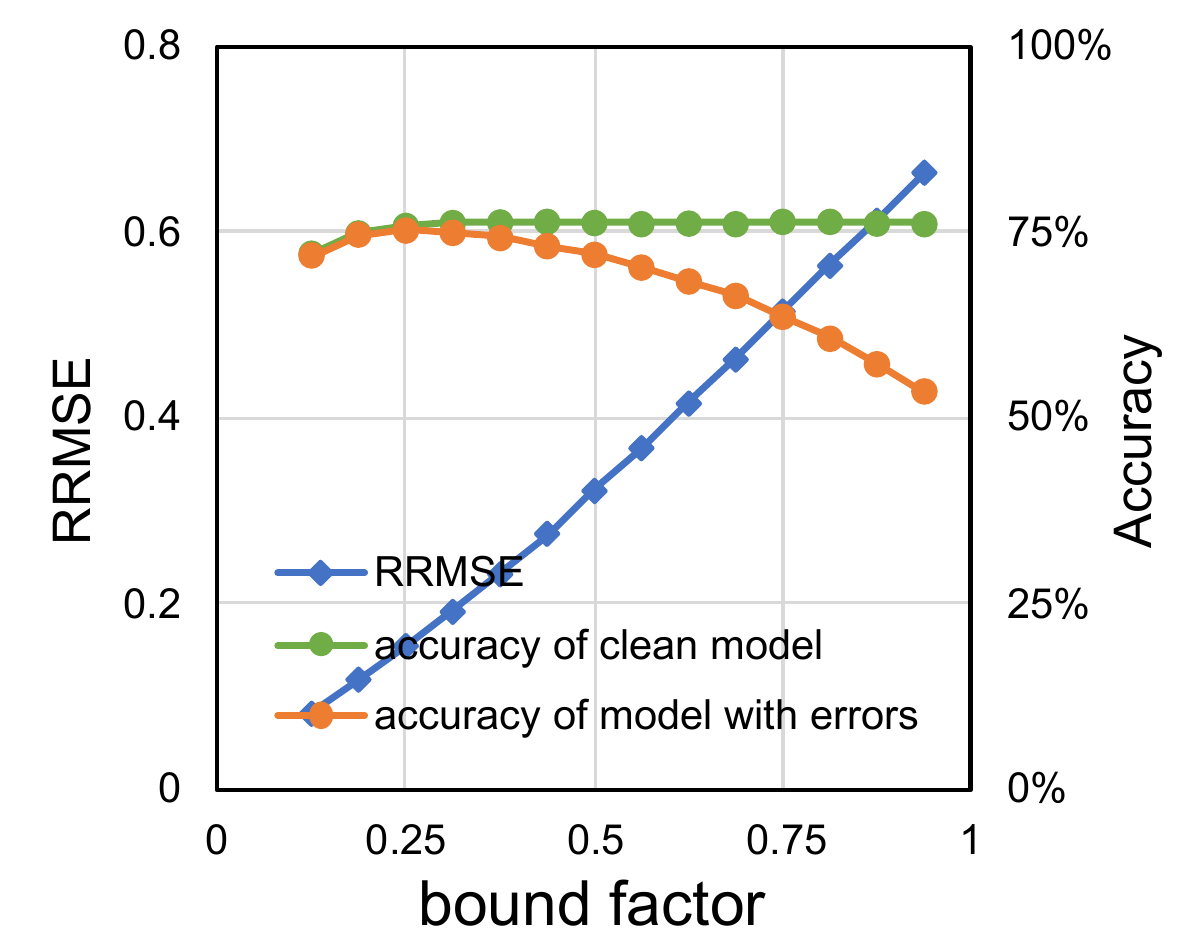}
    }
    \subfigure[ResNet-18]{
    \includegraphics[width=0.45\linewidth]{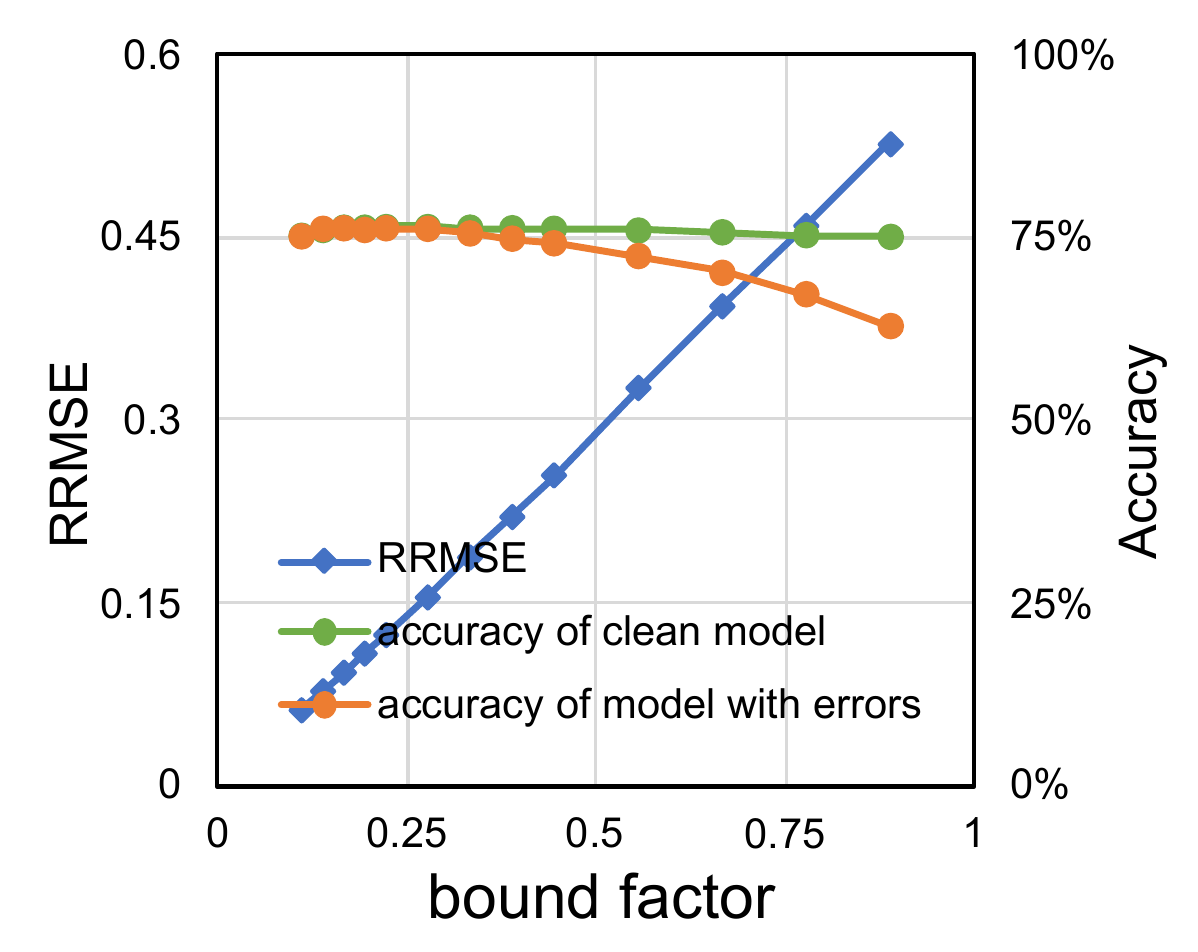}
    }
    \caption{Influence of quantization bound on RRMSE}
    \label{fig:bound_RMSE}
\end{figure}

\begin{figure}[htbp]
    \centering
    \subfigure[VGGNet-11]{
        \includegraphics[width=0.45\linewidth]{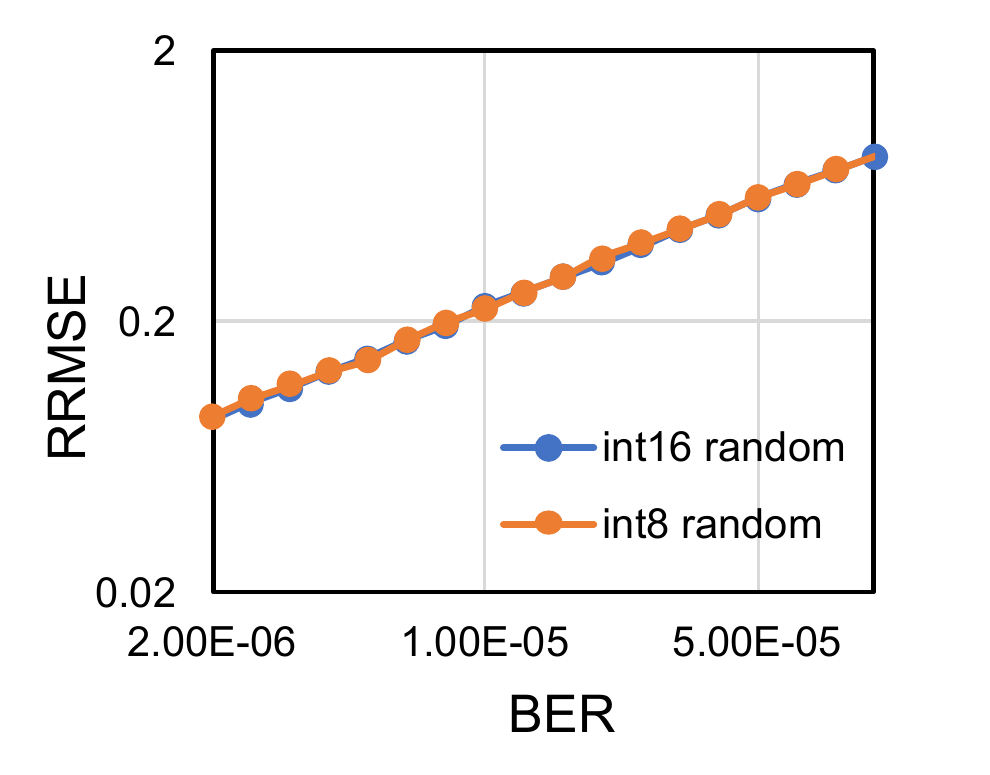}
    }
    \subfigure[ResNet-18]{
        \includegraphics[width=0.45\linewidth]{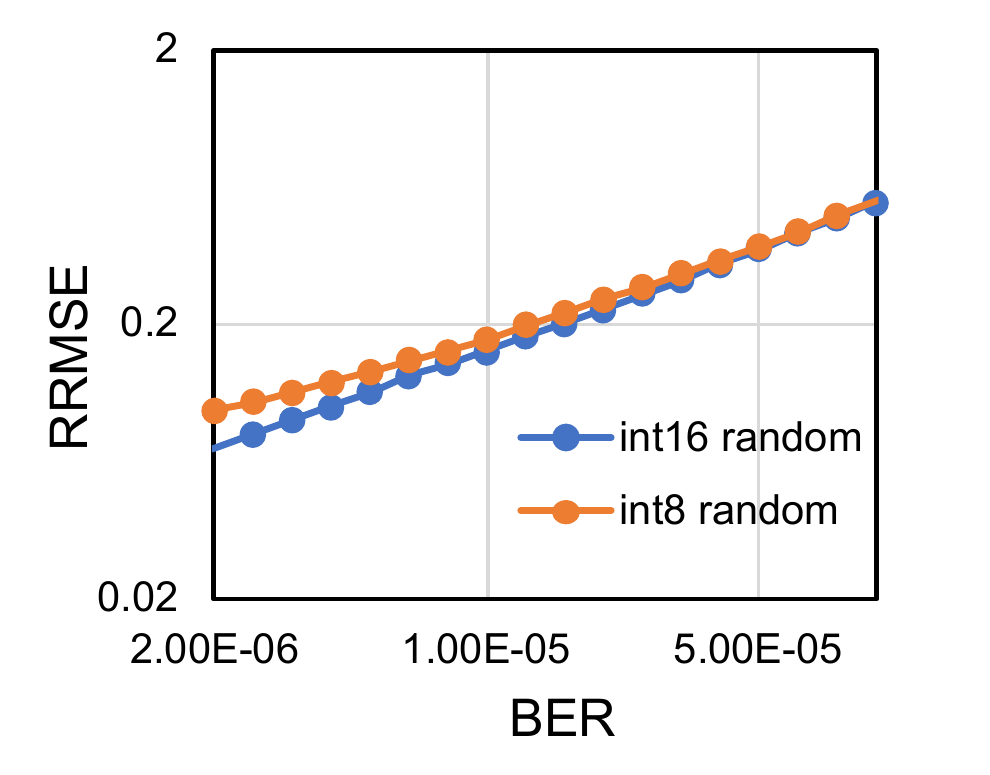}
    }
    \caption{Influence of quantizaton bitwidth on RRMSE}
    \label{fig:bitwidth_RMSE}
\end{figure}

According to the comparison in Equation  \ref{eq:int8-error-expectation} and Equation \ref{eq:int16-error-expectation}, we notice that the average disturbance induced by a single bit error for neural network model quantized with int16 is $\sqrt2\times$ smaller than that quantized with int8. On the other hand, the expected total number of bit errors for model quantized with int16 is twice larger than that quantized with int8 given the same bit error rate. According to Equation \ref{equ:combine}, $RRMSE$ of the same neural network model with more bit errors will be $\sqrt{2}\times$ larger. Hence, $RRMSE_{int8}$ of an neural network is equal to $RRMSE_{int16}$ in theory given the same bit error rate. Similar to prior analysis, we take VGGNet-11 and ResNet-18 quantized with int8 and int16 as examples and conduct fault simulation  to verify the model based analysis. The bit error rate is set to be 1E-6 and the quantization bound is set to be the same for both quantization data width. The experiment results shown in Figure \ref{fig:bitwidth_RMSE} demonstrate that RRMSE of neural network models are generally steady despite the quantization data width, which is consistent with the model based analysis.

\subsection{Influence of Classification Complexity on NN Resilience}\label{sec:analyze_classification}
% 网络的表现和具体任务关系密切。在以往的实验中，我们得到了一些模糊的、经验性的结论，例如网络在执行“更难”的任务时，受到注错的影响更大。这里的更难是我们直观上的感受，例如1000分类的任务一般比10分类的任务更难，做到Top-1正确比做到Top-5正确更难，对ImageNet图像数据集分类比mnist数据集0-9手写数字更难。在\cite{Bayesian2019}的贝叶斯网络实验中，作者也注意到分类边界上的点更容易受到注错的影响。

Intuitively, we notice that easier deep learning tasks are generally more resilient to errors, but it is usually difficult to define the complexity of a neural network processing task. Equation \ref{eq:accuracy-RMSE-multi} provides a model to characterize the relation between the number of classification types and classification accuracy. When we take the number of classification types as a metric of neural network complexity, it provides a simple yet efficient angle to characterize the relation between neural network complexity and neural network resilience. The model in Equation \ref{eq:accuracy-RMSE-multi}  proves that neural networks with more classification types are more vulnerable subjected to the same number of errors. To verify this, we take VGGNet-11 and ResNet-18 on ImageNet as examples and then configure them for a set of classification tasks with different number of classification types i.e. $nc$ ranging from 2 to 1000. Then, we explore the resulting accuracy of these classification tasks subjected to the same bit error setups. The experiment result is shown in Figure \ref{fig:model_RMSE_acc}. It reveals that neural network models with less classification types generally have much higher accuracy and the accuracy drops slower with increasing bit error rate compared to that with more classification types.

\begin{figure}[htbp]
    \centering
    \subfigure[VGGNet-11]{
    \includegraphics[width=0.46\linewidth]{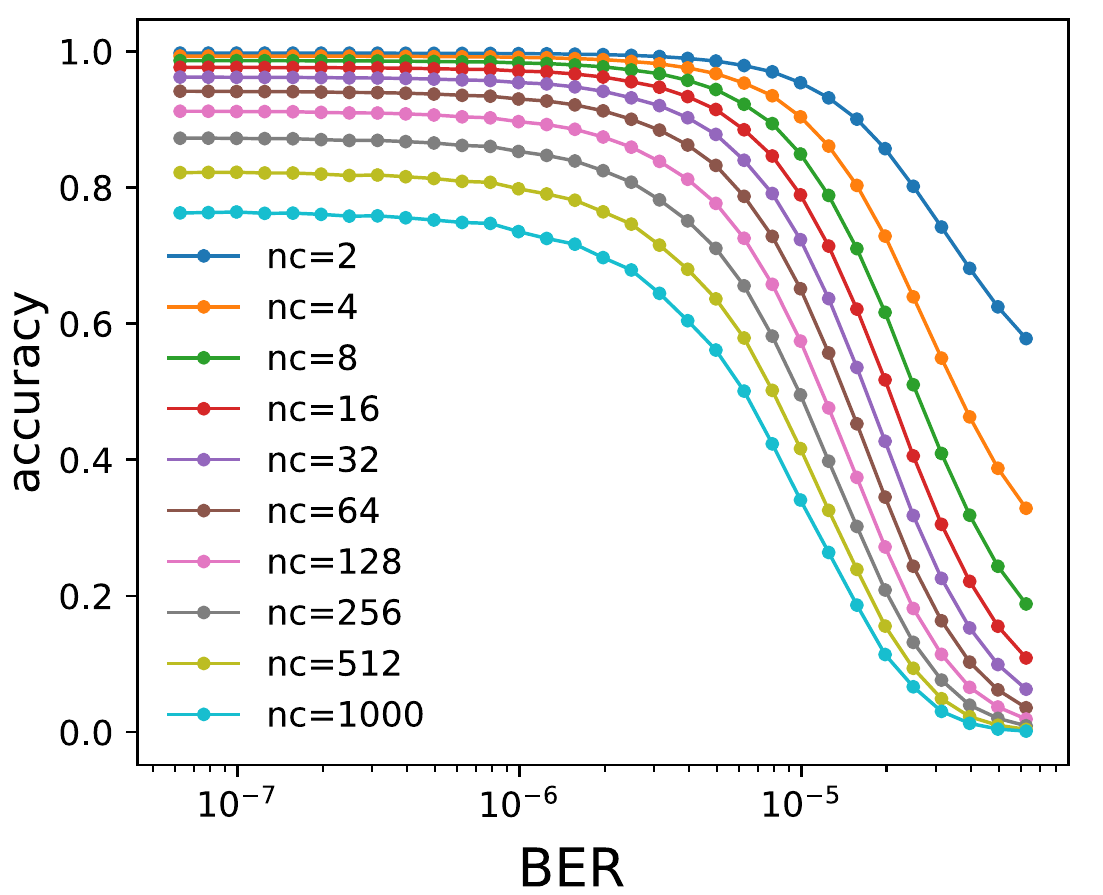}
    }
    \subfigure[ResNet-18]{
    \includegraphics[width=0.46\linewidth]{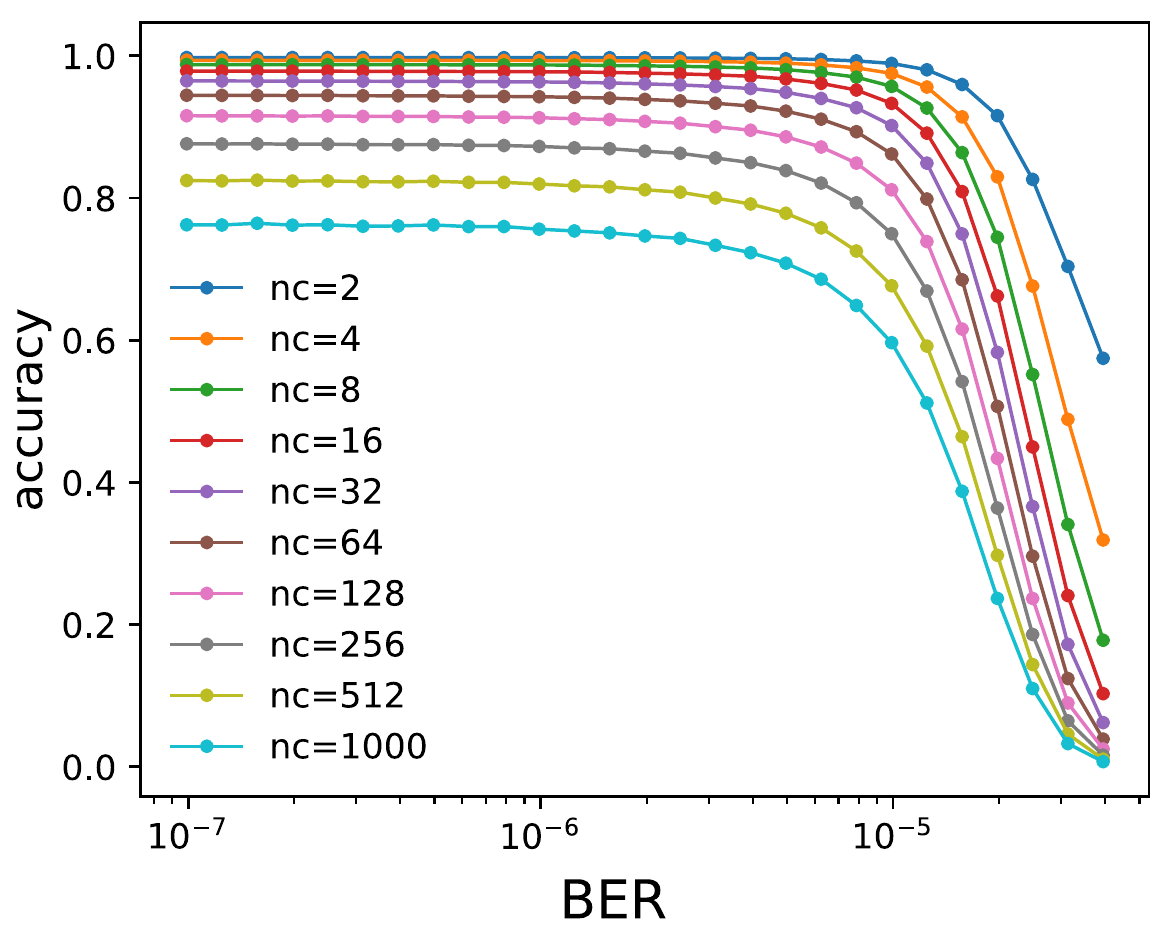}
    }
    \caption{VGG-11 classification accuracy under different number of classification types}
    \label{fig:model_RMSE_acc}
\end{figure}

\section{Modeling for Fault Simulation Acceleration}
Fault simulation is a typical practice for neural network resilience analysis under hardware errors and several fault simulation tools \cite{PytorchFI} \cite{TensorFI} \cite{FIdelity} targeting at neural networks have been developed with different trade-offs between simulation accuracy and speed. Usually, a large number of errors need to be injected and a variety of fault configurations need to be explored to ensure steady simulation results, which can be rather time-consuming and expensive. Orthogonal to prior fault simulation approaches, we mainly investigate how the modeling proposed in this work can be utilized to accelerate the fault simulation with negligible simulation accuracy loss.

\subsection{Fault Simulation Acceleration Approaches}
To begin, we will introduce three statistical model based approaches that can be utilized to accelerate general fault simulation. The basic idea is to leverage the proposed statistical models to predict the fault simulation results with only a fraction of simulation setups and reduce the number of fault simulation of all the possible fault configurations. The three acceleration approaches are listed as follows.

First, according to Equation \ref{eq:empirical-accuracy-RMSE-multi}, we can characterize the relation between RRMSE and model accuracy under various bit error rate setups with only five data points. Moreover, RRMSE of a model can be obtained with less input images and converges much faster than that of model accuracy. Particularly, we take VGGNet-11 on ImageNet as an example and investigate how the RRMSE and model accuracy of VGGNet-11 changes with different number of input image samples given the same bit error rate. The bit error rate is set to be 1E-6. The experiment result in Figure \ref{fig:coverge_speed} confirms the advantage of using RRMSE in terms of convergence. With both the RRMSE and the correlation curve of RRMSE and model accuracy, we can obtain the correlation curve of bit error rate and model accuracy much faster compared to conventional fault simulation experiments.

\begin{figure}[htbp]
    \centering
    \includegraphics[width=0.9\linewidth]{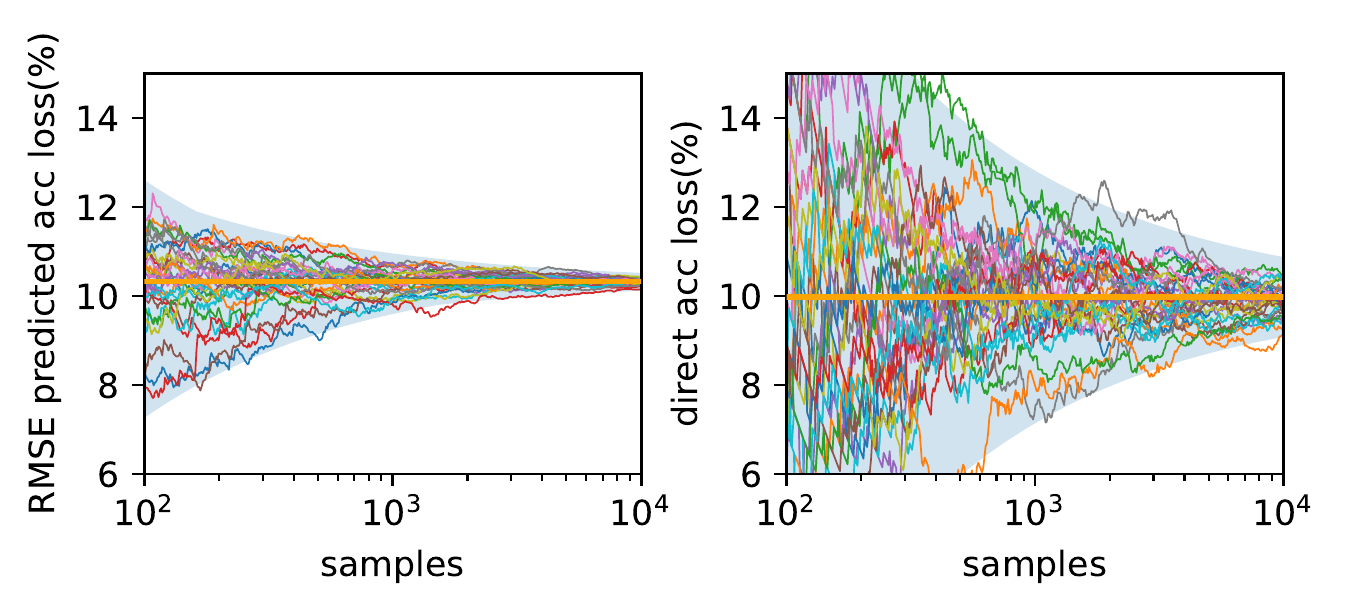}
    \caption{RRMSE(a) and model accuracy (b) obtained under different number of input images. Different lines refer to fault injection with the same bit error rate but on different random seed. It demonstrates that RRMSE converges much faster than model accuracy given more input images.
}
    \label{fig:coverge_speed}
\end{figure}

Second, with the analysis in Section \ref{sec:value_flip_model}, we notice that the error injection on different bits contributes differently to the output RRMSE eventually but the contribution proportion can be calculated based on Equation \ref{eq:int8-error-expectation} and Equation \ref{eq:int16-error-expectation}. Similarly, we can also obtain the expected data disturbance of bit error on MSB. Then, we can calculate $RRMSE_{int8}$ based on $RRMSE_{MSB}$ with Equation \ref{eq:rrmse-scale-int8} and \ref{eq:rrmse-scale-int16}, and replace the standard random bit error injection with most significant bit (MSB) based error injection without compromising the analysis accuracy. Take VGGNet-11 and ResNet-18 quantized with int16 as examples. Figure \ref{fig:exp_rate_rmse} shows the RRMSE obtained with both standard random bit error injection and MSB based error injection. It confirms that the resulting RRMSE obtained with standard error injection and MSB based error injection are linearly correlated as analyzed with the proposed statistical models. This approach can also be applied for straightforward model accuracy simulation. MSB based fault simulation and standard fault simulation for both VGGNet-11 and ResNet-18 is presented in Figure \ref{fig:modelwise_rate_acc}. It can be observed that the curves are quite similar and the difference is mainly induced by the scaled bit error rate. In summary, given the same bit error rate, MSB based error injection can be scaled for standard bit error injection with negligible accuracy penalty while it reduces the total number of injected bit errors substantially and enhances the fault simulation speed accordingly.
\begin{equation} \label{eq:rrmse-scale-int8}
    RRMSE_{MSB}=\sqrt{6}RRMSE_{int8}
\end{equation}

\begin{equation}
\label{eq:rrmse-scale-int16}
    RRMSE_{MSB}=\sqrt{12}RRMSE_{int16}
\end{equation}

\begin{figure}[htbp]
    \centering
    \subfigure[VGGNet-11]{
        \includegraphics[width=0.45\linewidth]{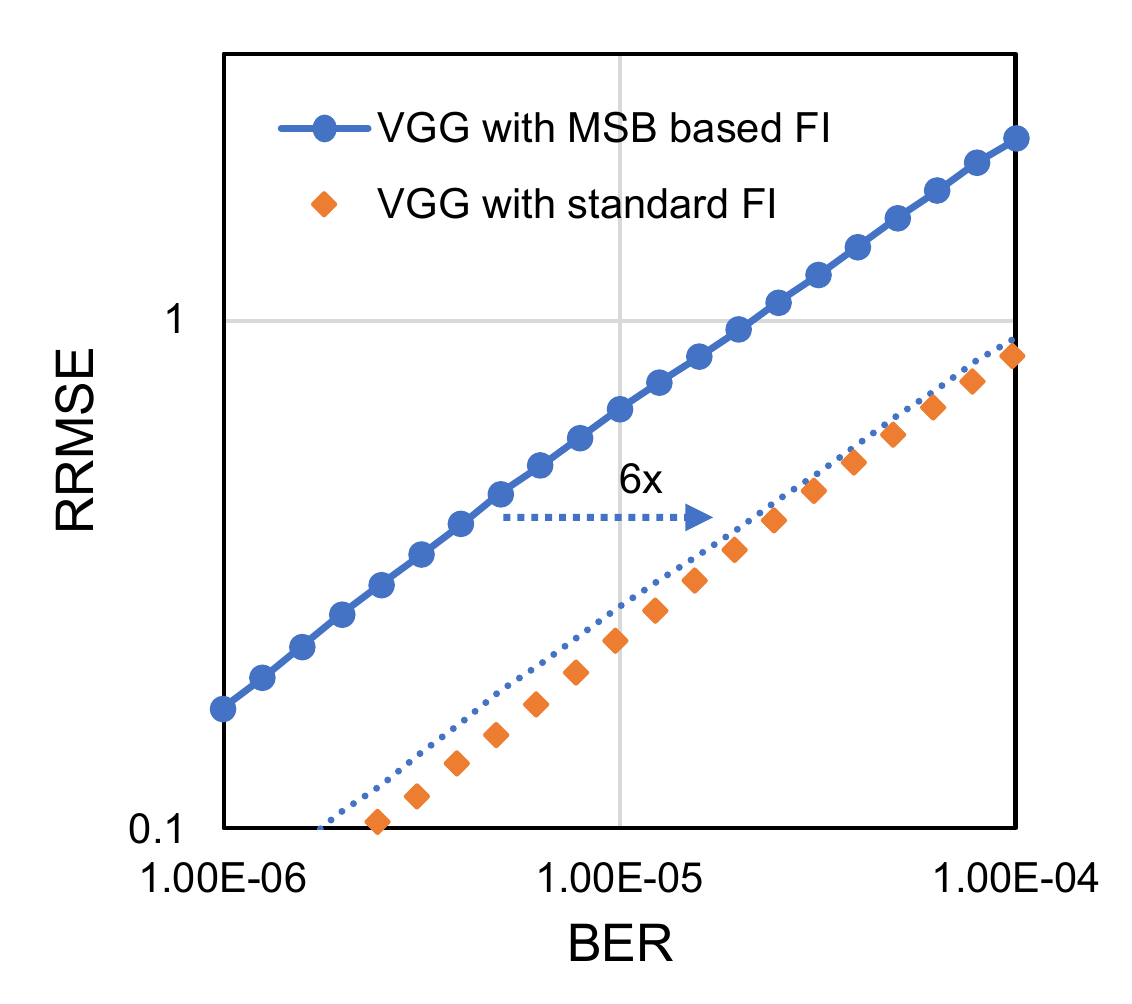}
    }
    \subfigure[ResNet-18]{
        \includegraphics[width=0.45\linewidth]{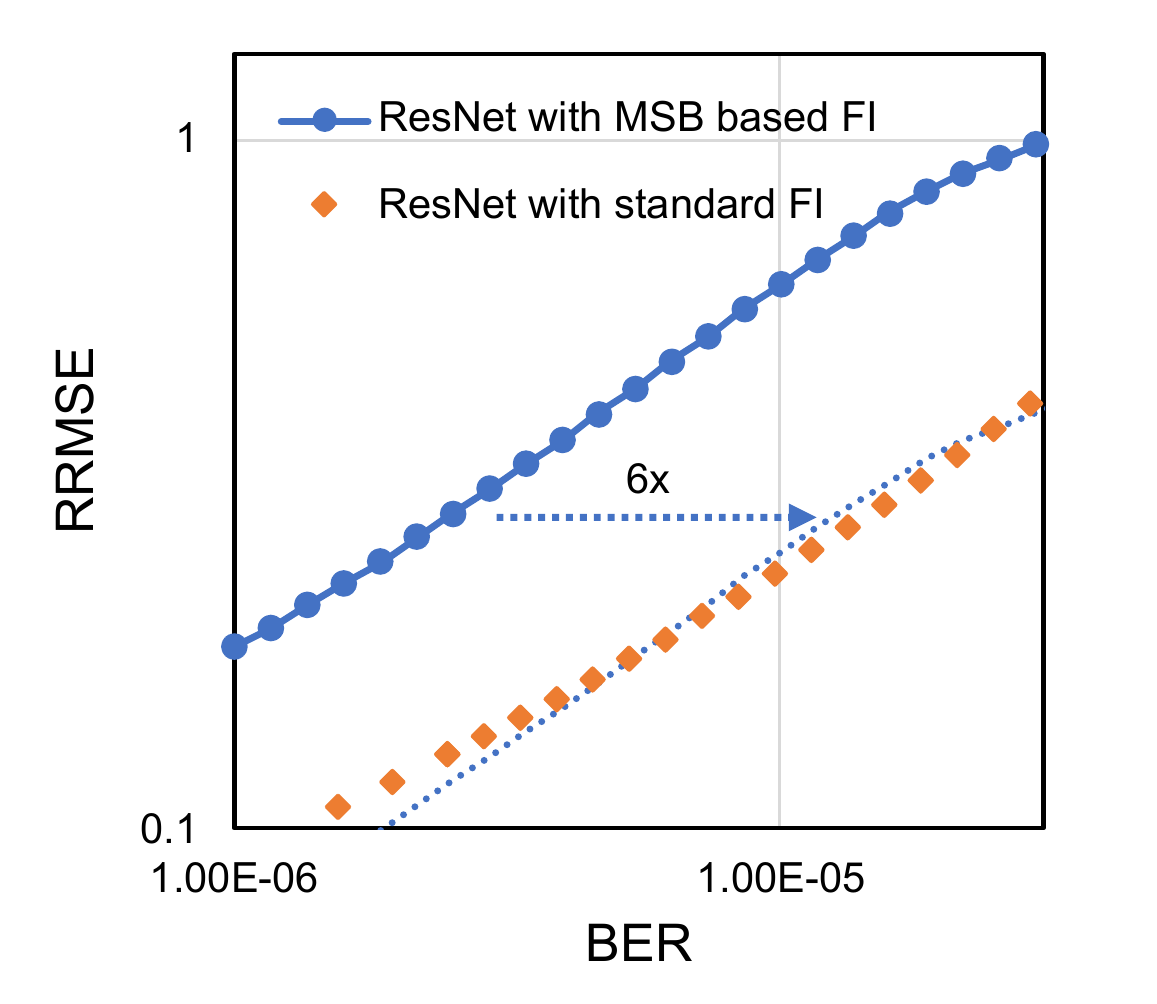}
    }
    \caption{RRMSE of VGGNet-11 and ResNet-18 obtained with both standard fault simulation and MSB based fault simulation.}
    \label{fig:exp_rate_rmse}
\end{figure}

\begin{figure}[htbp]
    \centering
    \subfigure[VGGNet-11]{
        \includegraphics[width=0.45\linewidth]{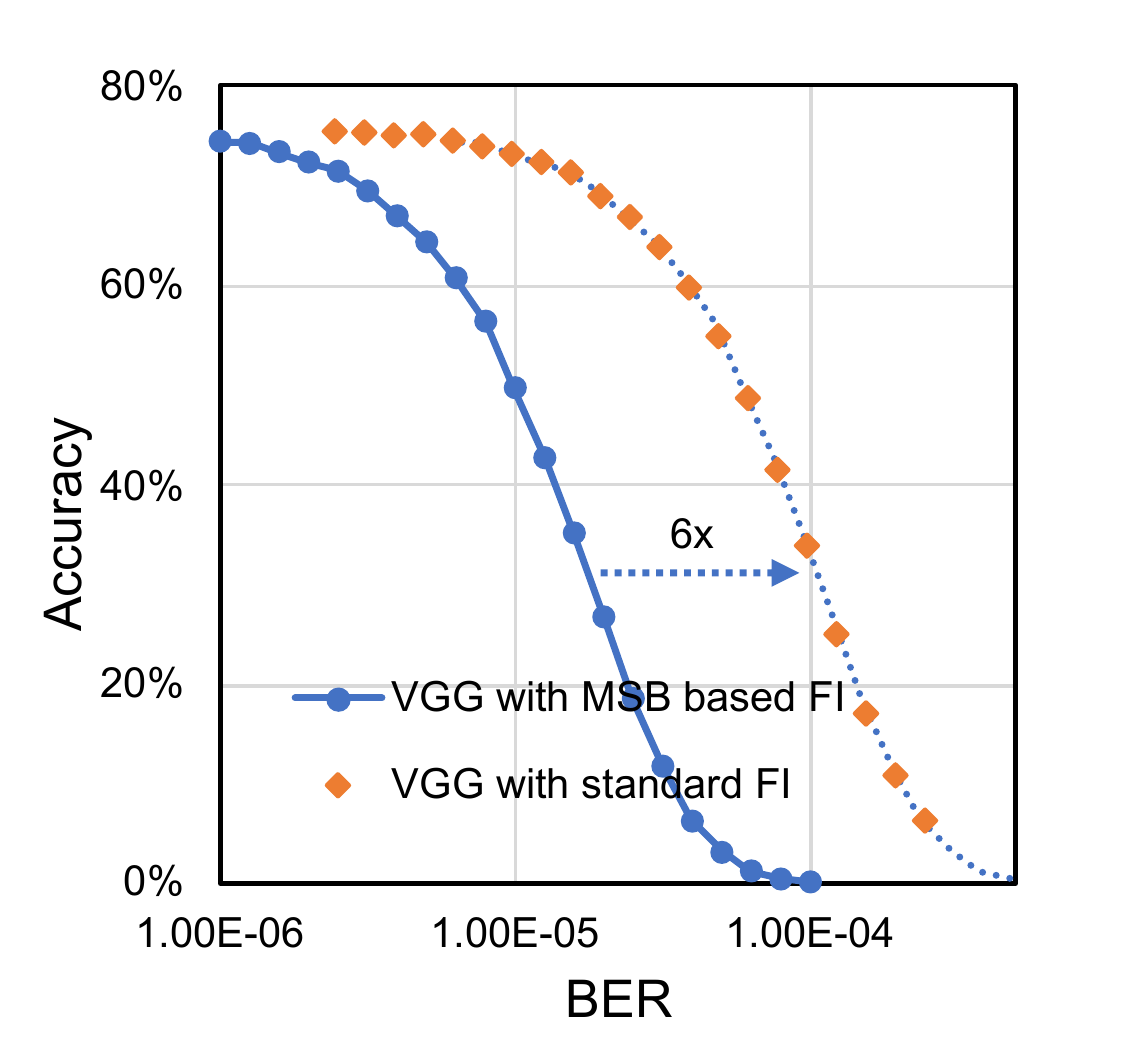}
    }
    \subfigure[ResNet-18]{
        \includegraphics[width=0.45\linewidth]{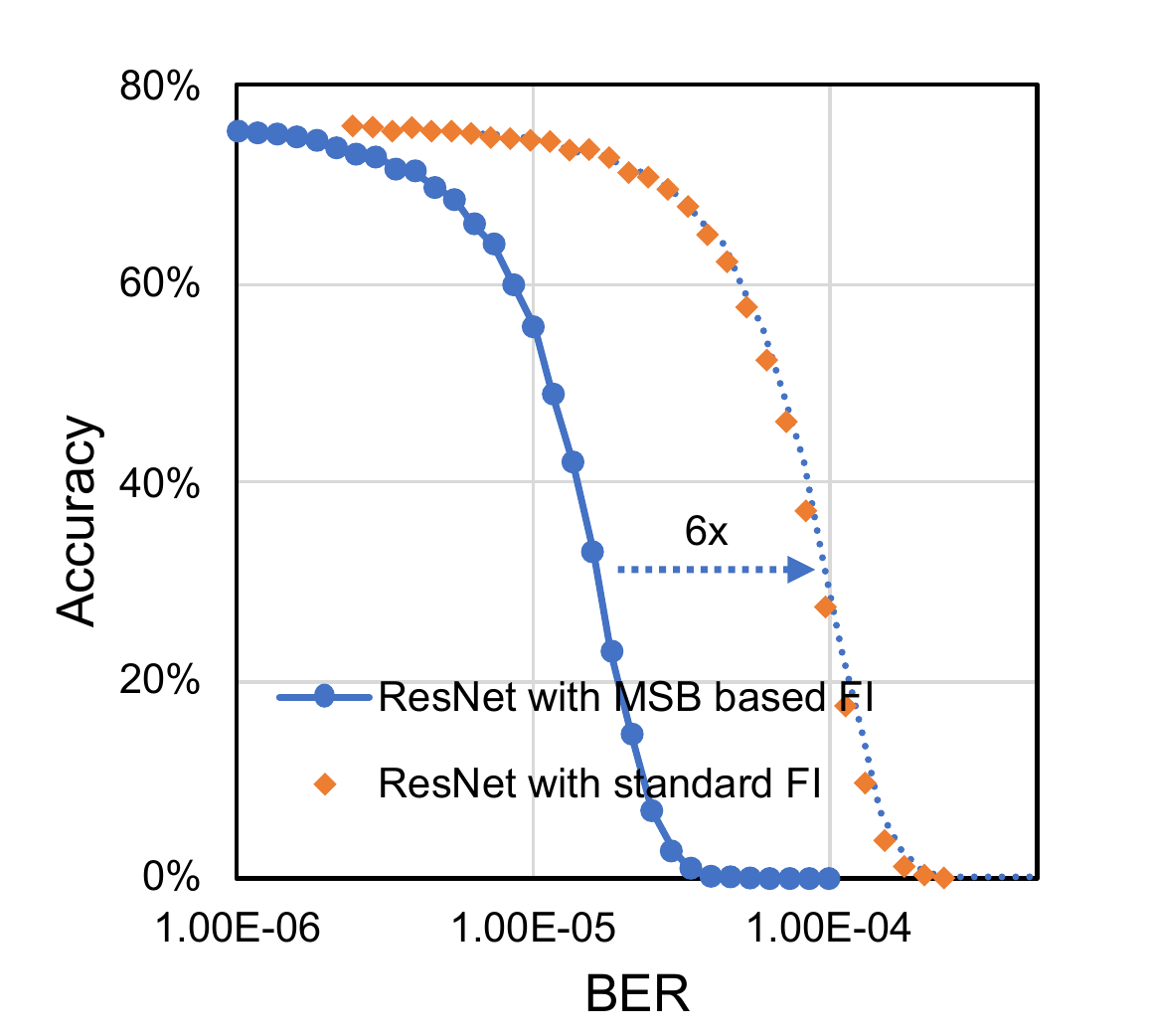}
    }
    \caption{Model accuracy of VGGNet-11 and ResNet-18 obtained with standard bit error injection and MSB based bit error injection.}
    \label{fig:modelwise_rate_acc}
\end{figure}

Third, according to the analysis in Section \ref{sec:estimate_directly}, we notice that the influence of bit errors in different neurons and layers can be aggregated with Equation \ref{equ:combine} when we utilize RRMSE as the accuracy metric. Based on this feature, we can analyze the influence of complex error configurations with a disaggregated approach. For instance, we can analyze the influence of random bit errors in each layer independently and then construct the influence of random bit errors on multiple neural network layers. We can also leverage the aggregation feature to scale RRMSE at lower bit error rate to RRMSE at higher bit error rate. In general, we can take advantage of this feature to obtain fault simulation of complex and large fault configurations with only a fraction of the fault configurations, which can greatly reduce the amount of fault simulation.

\subsection{Fault Simulation Acceleration Examples}
To quantize the fault simulation speedup, we take two typical fault simulation tasks as examples. 

In the first task, we investigate how the model accuracy changes with the increase of bit error rate. We still utilize VGGNet-11 quantized with int16 on ImageNet as the benchmark model. Suppose we want to explore the model accuracy when the bit error rate changes from 1E-7 to 1E-4 and we take 32 evenly distributed bit error rate setups for the experiments. For each bit error rate setup, we take 10000 images from ImageNet to measure the model accuracy. With TensorFI fault injection framework \cite{TensorFI}, we need to conduct $3.2\times10^6$ inference with bit error injection. The fault simulation task can be accelerated with the first and the second approaches. It needs only 4 fault simulation points of RRMSE and model accuracy with $16\times$ less bit error injection. The fault simulation time can be roughly reduced by $128 \times$. The results obtained with both a standard fault simulation and the accelerated fault simulation are shown in \autoref{fig:acclerate_simulation}. They are quite close to each other, which confirms the quality of the accelerated fault simulation.  

\begin{figure}[htbp]
    \centering
    \includegraphics[width=0.6\linewidth]{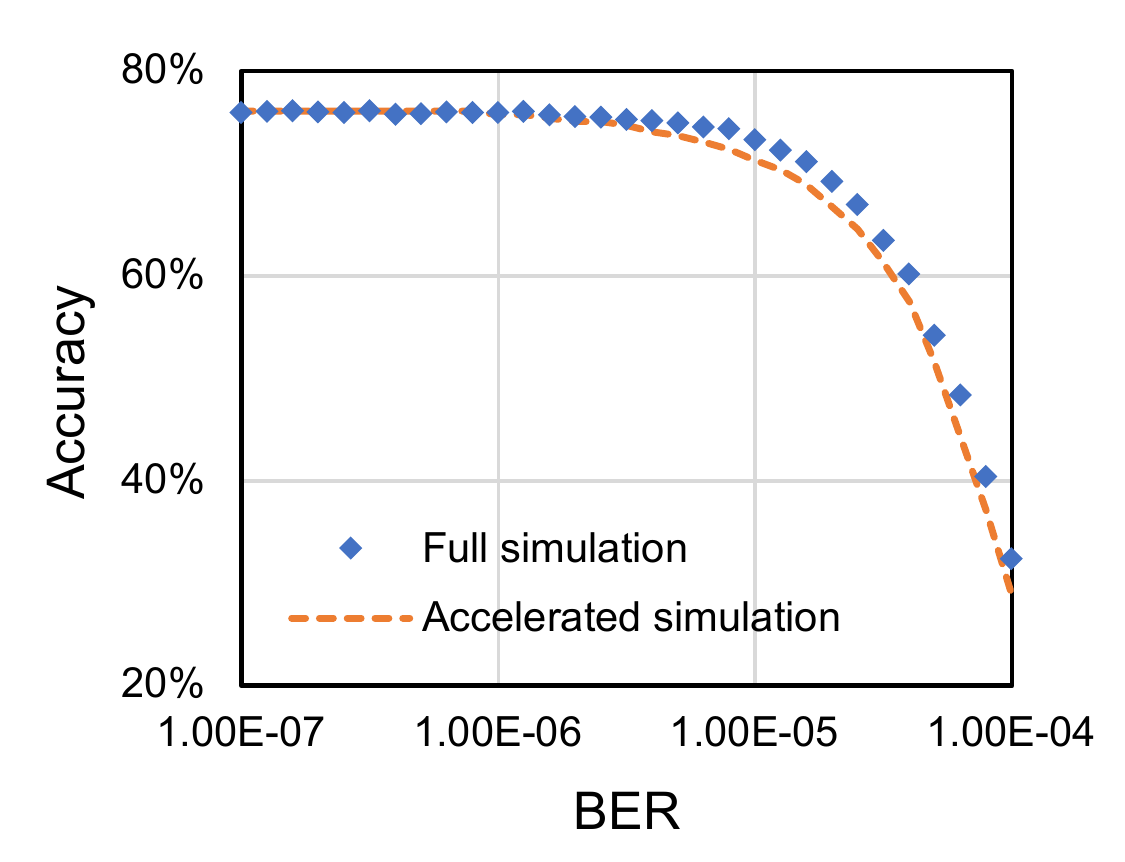}
    \caption{Model accuracy obtained by full simulation and accelerated simulation.}
    \label{fig:acclerate_simulation}
\end{figure}

%对于容错能力分析加速，可以采用MSB注入方法，并且用拟合公式减少注错点位数。不过由于PyTorch在软件层面的注错额外开销已经很小，MSB几乎不会带来额外的性能提高(但对于单次注错很耗时的注错器非常有效)。拟合与采样点数量有关。

%对于敏感层选择，可以用RRMSE代替acc作为指标，减少10x的实验次数。由于前人的工作中已经使用了类似的思想，所以我们这里只是给出更好的理论依据。

%对于多层组合的分析，使用各层分别组合可以得到指数级加速，前面的Error Influence Aggregation节中的实验已经可以体现效果。虽然这对于在模型层面选出容错能力最后，但是可以设想这样一个场景应用：在硬件上各个层有多种保护力度可以选择并且有不同的开销，以往对于各种方案需要指数多次实验，现在可以直接分析各种选择下容错能力。

In the second task, we seek to find out the top-3 most fragile layers of a neural network model such that an selectively hardening approach can be applied to protect the neural network model with less protection overhead. We take VGGNet-11 quantized with int16 as a benchmark example and conduct fault injection at $BER=5\times 10^{-5}$. There are 8 convolution layers in VGGNet-11. A standard fault simulation based method needs to evaluate all the $C^3_8=56$ different combinations. For each configuration, we need to perform inference on around 10000 images. Theoretically, we need to conduct $5.6 \times 10^5$ inference with random bit error injection. In contrast, with the second and the third acceleration approaches, we only need to conduct 8 layer-wise MSB based error injection and perform inference on 1000 images for each error injection setup to obtain the neural network RRMSE. Then, we can obtain RRMSE of all the different layer combinations immediately according to Equation \ref{equ:combine} in Section \ref{sec:modeling}. In this case, the fault simulation time can be reduced by $560\times$. For neural networks with more layers, this method can achieve exponential acceleration. In addition, we also evaluate the selected top-3 fragile layers based on the accelerated fault simulation approach. Suppose the top-3 fragile layers are fault-free, we can obtain the model accuracy with standard fault simulation and compare with that of all the 56 different combinations. The experiment result is shown in \autoref{fig:vgg11_comb_accel}. It demonstrates that the selected top-3 layer are the most fragile layers of VGGNet-11, which is consistent with the results of standard fault simulation. In summary, the accelerated fault simulation not only reduces the execution time but also achieves high-quality results on this vulnerability analysis task.

% It is worth noting that our analysis results can be easily extended to more configurations without any more simulation. For example, when the model is deployed to different hardware, there may be different bit error rates and quantization methods. Our method does not need time-consuming simulation again. This helps designers to quickly verify whether the fault tolerance of the changed model meets the requirements.
% }

\begin{figure}
    \centering
    \includegraphics[width=0.6\linewidth]{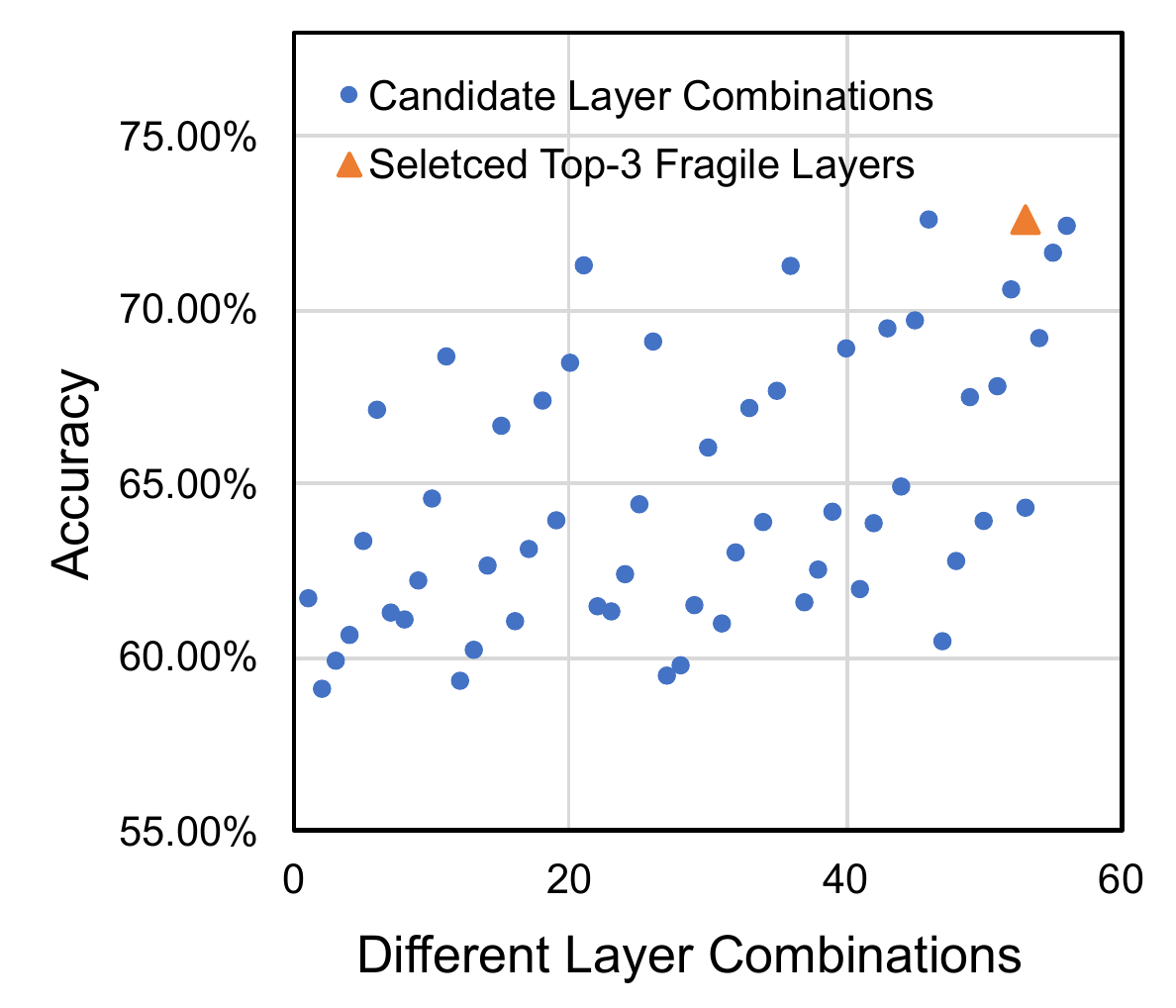}
    \caption{We have the top-3 fragile layers selected with the proposed fault simulation approach evaluated and compare with all the 56 different candidate combinations. We have the accuracy of VGGNet-11 under error injection used as the metric and assume the evaluated 3 convolution layers of VGGNet-11 are set to be fault-free.}
    \label{fig:vgg11_comb_accel}
\end{figure}

\section{Conclusion}
In this work, we observe that soft errors propagate across a large number of neurons and accumulate. With the observation, we propose to characterize the disturbance induced by the soft errors with a normal distribution model according to central limit theorem and analyze the influence of soft errors on neural networks with a series of statistical models. The models are further applied to  analyze the influence of convolutional neural network parameters on its resilience and verified with experiments comprehensively, which can guide the fault-tolerant neural network design. In addition, we find that the models can also be utilized to predict many fault simulation results precisely and avoid lengthy fault simulation in practice. Given two typical fault simulation tasks, the model-accelerated fault simulation can be more than two orders of magnitude faster on average with negligible accuracy loss compared to standard fault simulation.

\bibliographystyle{plain}
\bibliography{cite}

\begin{thebibliography}{10}

\bibitem{Weight2018}
Muhammad Atta~Othman Ahmed.
\newblock {Trained Neural Networks Ensembles Weight Connections Analysis}.
\newblock {\em Advances in Intelligent Systems and Computing},
  723(January):242--251, 2018.

\bibitem{Bayesian2019}
Subho~S. Banerjee, James Cyriac, Saurabh Jha, Zbigniew~T. Kalbarczyk, and
  Ravishankar~K. Iyer.
\newblock {Towards a Bayesian Approach for Assessing Fault Tolerance of Deep
  Neural Networks}.
\newblock {\em Proceedings - 49th Annual IEEE/IFIP International Conference on
  Dependable Systems and Networks - Supplemental Volume, DSN-S 2019}, pages
  25--26, 2019.

\bibitem{chen2019BinFI}
Zitao Chen, Guanpeng Li, Karthik Pattabiraman, and Nathan DeBardeleben.
\newblock Binfi: An efficient fault injector for safety-critical machine
  learning systems.
\newblock In {\em Proceedings of the International Conference for High
  Performance Computing, Networking, Storage and Analysis}, SC '19, New York,
  NY, USA, 2019. Association for Computing Machinery.

\bibitem{TensorFI}
Zitao Chen, Niranjhana Narayanan, Bo~Fang, Guanpeng Li, Karthik Pattabiraman,
  and Nathan DeBardeleben.
\newblock {Tensorfi: A flexible fault injection framework for tensorflow
  applications}.
\newblock {\em Proceedings - International Symposium on Software Reliability
  Engineering, ISSRE}, 2020-Octob:426--435, 2020.

\bibitem{dixit2011impact}
Anand Dixit and Alan Wood.
\newblock The impact of new technology on soft error rates.
\newblock In {\em 2011 International Reliability Physics Symposium}, pages
  5B--4. IEEE, 2011.

\bibitem{dos2019reliability}
Fernando~Fernandes dos Santos, Caio Lunardi, Daniel Oliveira, Fabiano Libano,
  and Paolo Rech.
\newblock Reliability evaluation of mixed-precision architectures.
\newblock In {\em 2019 IEEE International Symposium on High Performance
  Computer Architecture (HPCA)}, pages 238--249. IEEE, 2019.

\bibitem{Weight2020}
Gianni Franchi, Andrei Bursuc, Emanuel Aldea, S{\'{e}}verine Dubuisson, and
  Isabelle Bloch.
\newblock {TRADI: Tracking Deep Neural Network Weight Distributions}.
\newblock {\em Lecture Notes in Computer Science (including subseries Lecture
  Notes in Artificial Intelligence and Lecture Notes in Bioinformatics)}, 12362
  LNCS:105--121, 2020.

\bibitem{gambardella2019efficient}
Giulio Gambardella, Johannes Kappauf, Michaela Blott, Christoph Doehring,
  Martin Kumm, Peter Zipf, and Kees Vissers.
\newblock Efficient error-tolerant quantized neural network accelerators.
\newblock In {\em 2019 IEEE International Symposium on Defect and Fault
  Tolerance in VLSI and Nanotechnology Systems (DFT)}, pages 1--6. IEEE, 2019.

\bibitem{gao2022soft}
Zhen Gao, Han Zhang, Yi~Yao, Jiajun Xiao, Shulin Zeng, Guangjun Ge, Yu~Wang,
  Anees Ullah, and Pedro Reviriego.
\newblock Soft error tolerant convolutional neural networks on fpgas with
  ensemble learning.
\newblock {\em IEEE Transactions on Very Large Scale Integration (VLSI)
  Systems}, 30(3):291--302, 2022.

\bibitem{SASSIFI}
Siva Kumar~Sastry Hari, Timothy Tsai, Mark Stephenson, Stephen~W. Keckler, and
  Joel Emer.
\newblock {SASSIFI: An architecture-level fault injection tool for GPU
  application resilience evaluation}.
\newblock {\em ISPASS 2017 - IEEE International Symposium on Performance
  Analysis of Systems and Software}, 1(1):249--258, 2017.

\bibitem{hashimoto2018artificial}
Daniel~A Hashimoto, Guy Rosman, Daniela Rus, and Ozanan~R Meireles.
\newblock Artificial intelligence in surgery: promises and perils.
\newblock {\em Annals of surgery}, 268(1):70, 2018.

\bibitem{FIdelity}
Yi~He, Prasanna Balaprakash, and Yanjing Li.
\newblock {Fidelity: Efficient resilience analysis framework for deep learning
  accelerators}.
\newblock {\em Proceedings of the Annual International Symposium on
  Microarchitecture, MICRO}, 2020-Octob:270--281, 2020.

\bibitem{humbatova2020taxonomy}
Nargiz Humbatova, Gunel Jahangirova, Gabriele Bavota, Vincenzo Riccio, Andrea
  Stocco, and Paolo Tonella.
\newblock Taxonomy of real faults in deep learning systems.
\newblock In {\em Proceedings of the ACM/IEEE 42nd International Conference on
  Software Engineering}, pages 1110--1121, 2020.

\bibitem{Li2017understanding}
Guanpeng Li, Siva Kumar~Sastry Hari, Michael~B. Sullivan, Timothy Tsai, Karthik
  Pattabiraman, Joel~S. Emer, and Stephen~W. Keckler.
\newblock Understanding error propagation in deep learning neural network
  {(DNN)} accelerators and applications.
\newblock In Bernd Mohr and Padma Raghavan, editors, {\em Proceedings of the
  International Conference for High Performance Computing, Networking, Storage
  and Analysis, {SC} 2017, Denver, CO, USA, November 12 - 17, 2017}, pages
  8:1--8:12. {ACM}, 2017.

\bibitem{li2020ftt}
Wenshuo Li, Xuefei Ning, Guangjun Ge, Xiaoming Chen, Yu~Wang, and Huazhong
  Yang.
\newblock Ftt-nas: Discovering fault-tolerant neural architecture.
\newblock In {\em 2020 25th Asia and South Pacific Design Automation Conference
  (ASP-DAC)}, pages 211--216. IEEE, 2020.

\bibitem{libano2018selective}
Fabiano Libano, Brittany Wilson, J~Anderson, Michael~J Wirthlin, Carlo
  Cazzaniga, Christopher Frost, and Paolo Rech.
\newblock Selective hardening for neural networks in fpgas.
\newblock {\em IEEE Transactions on Nuclear Science}, 66(1):216--222, 2018.

\bibitem{liu2021hyca}
Cheng Liu, Cheng Chu, Dawen Xu, Ying Wang, Qianlong Wang, Huawei Li, Xiaowei
  Li, and Kwang-Ting Cheng.
\newblock Hyca: A hybrid computing architecture for fault tolerant deep
  learning.
\newblock {\em IEEE Transactions on Computer-Aided Design of Integrated
  Circuits and Systems}, 2021.

\bibitem{liu2022special}
Cheng Liu, Zhen Gao, Siting Liu, Xuefei Ning, Huawei Li, and Xiaowei Li.
\newblock Special session: Fault-tolerant deep learning: A hierarchical
  perspective.
\newblock In {\em 2022 IEEE 40th VLSI Test Symposium (VTS)}, pages 1--12. IEEE,
  2022.

\bibitem{PytorchFI}
Abdulrahman Mahmoud, Neeraj Aggarwal, Alex Nobbe, Jose Rodrigo~Sanchez Vicarte,
  Sarita~V. Adve, Christopher~W. Fletcher, Iuri Frosio, and Siva Kumar~Sastry
  Hari.
\newblock {PyTorchFI: A Runtime Perturbation Tool for DNNs}.
\newblock {\em Proceedings - 50th Annual IEEE/IFIP International Conference on
  Dependable Systems and Networks, DSN-W 2020}, pages 25--31, 2020.

\bibitem{mittal2020survey}
Sparsh Mittal.
\newblock A survey on modeling and improving reliability of dnn algorithms and
  accelerators.
\newblock {\em Journal of Systems Architecture}, 104:101689, 2020.

\bibitem{muhammad2020deep}
Khan Muhammad, Amin Ullah, Jaime Lloret, Javier Del~Ser, and Victor Hugo~C
  de~Albuquerque.
\newblock Deep learning for safe autonomous driving: Current challenges and
  future directions.
\newblock {\em IEEE Transactions on Intelligent Transportation Systems},
  22(7):4316--4336, 2020.

\bibitem{ning2021ftt}
Xuefei Ning, Guangjun Ge, Wenshuo Li, Zhenhua Zhu, Yin Zheng, Xiaoming Chen,
  Zhen Gao, Yu~Wang, and Huazhong Yang.
\newblock Ftt-nas: Discovering fault-tolerant convolutional neural
  architecture.
\newblock {\em ACM Transactions on Design Automation of Electronic Systems
  (TODAES)}, 26(6):1--24, 2021.

\bibitem{o2019legal}
Shane O'Sullivan, Nathalie Nevejans, Colin Allen, Andrew Blyth, Simon Leonard,
  Ugo Pagallo, Katharina Holzinger, Andreas Holzinger, Mohammed~Imran Sajid,
  and Hutan Ashrafian.
\newblock Legal, regulatory, and ethical frameworks for development of
  standards in artificial intelligence (ai) and autonomous robotic surgery.
\newblock {\em The international journal of medical robotics and computer
  assisted surgery}, 15(1):e1968, 2019.

\bibitem{SNR2021}
Elbruz Ozen and Alex Orailoglu.
\newblock {SNR: Squeezing Numerical Range Defuses Bit Error Vulnerability
  Surface in Deep Neural Networks}.
\newblock {\em ACM Trans. Embed. Comput. Syst.}, 20(5s), 2021.

\bibitem{pandey2019greentpu}
Pramesh Pandey, Prabal Basu, Koushik Chakraborty, and Sanghamitra Roy.
\newblock Greentpu: Improving timing error resilience of a near-threshold
  tensor processing unit.
\newblock In {\em 2019 56th ACM/IEEE Design Automation Conference (DAC)}, pages
  1--6. IEEE, 2019.

\bibitem{pouyanfar2018survey}
Samira Pouyanfar, Saad Sadiq, Yilin Yan, Haiman Tian, Yudong Tao, Maria~Presa
  Reyes, Mei-Ling Shyu, Shu-Ching Chen, and Sundaraja~S Iyengar.
\newblock A survey on deep learning: Algorithms, techniques, and applications.
\newblock {\em ACM Computing Surveys (CSUR)}, 51(5):1--36, 2018.

\bibitem{Ares2018}
Brandon Reagen, Udit Gupta, Lillian Pentecost, Paul Whatmough, Sae~Kyu Lee,
  Niamh Mulholland, David Brooks, and Gu~Yeon Wei.
\newblock {Ares: A framework for quantifying the resilience of deep neural
  networks}.
\newblock In {\em Proceedings - Design Automation Conference}, volume Part
  F1377, 2018.

\bibitem{reagen2016minerva}
Brandon Reagen, Paul Whatmough, Robert Adolf, Saketh Rama, Hyunkwang Lee,
  Sae~Kyu Lee, Jos{\'e}~Miguel Hern{\'a}ndez-Lobato, Gu-Yeon Wei, and David
  Brooks.
\newblock Minerva: Enabling low-power, highly-accurate deep neural network
  accelerators.
\newblock In {\em 2016 ACM/IEEE 43rd Annual International Symposium on Computer
  Architecture (ISCA)}, pages 267--278. IEEE, 2016.

\bibitem{reuther2019survey}
Albert Reuther, Peter Michaleas, Michael Jones, Vijay Gadepally, Siddharth
  Samsi, and Jeremy Kepner.
\newblock Survey and benchmarking of machine learning accelerators.
\newblock In {\em 2019 IEEE high performance extreme computing conference
  (HPEC)}, pages 1--9. IEEE, 2019.

\bibitem{reuther2021ai}
Albert Reuther, Peter Michaleas, Michael Jones, Vijay Gadepally, Siddharth
  Samsi, and Jeremy Kepner.
\newblock Ai accelerator survey and trends.
\newblock In {\em 2021 IEEE High Performance Extreme Computing Conference
  (HPEC)}, pages 1--9. IEEE, 2021.

\bibitem{shafique2020robust}
Muhammad Shafique, Mahum Naseer, Theocharis Theocharides, Christos Kyrkou, Onur
  Mutlu, Lois Orosa, and Jungwook Choi.
\newblock Robust machine learning systems: Challenges, current trends,
  perspectives, and the road ahead.
\newblock {\em IEEE Design \& Test}, 37(2):30--57, 2020.

\bibitem{torres2017fault}
Cesar Torres-Huitzil and Bernard Girau.
\newblock Fault and error tolerance in neural networks: A review.
\newblock {\em IEEE Access}, 5:17322--17341, 2017.

\bibitem{xu2021r2f}
Dawen Xu, Meng He, Cheng Liu, Ying Wang, Long Cheng, Huawei Li, Xiaowei Li, and
  Kwang-Ting Cheng.
\newblock R2f: A remote retraining framework for aiot processors with computing
  errors.
\newblock {\em IEEE Transactions on Very Large Scale Integration (VLSI)
  Systems}, 29(11):1955--1966, 2021.

\bibitem{xu2020persistent}
Dawen Xu, Ziyang Zhu, Cheng Liu, Ying Wang, Huawei Li, Lei Zhang, and
  Kwang-Ting Cheng.
\newblock Persistent fault analysis of neural networks on fpga-based
  acceleration system.
\newblock In {\em 2020 IEEE 31st International Conference on
  Application-specific Systems, Architectures and Processors (ASAP)}, pages
  85--92. IEEE, 2020.

\bibitem{xu2021reliability}
Dawen Xu, Ziyang Zhu, Cheng Liu, Ying Wang, Shuang Zhao, Lei Zhang, Huaguo
  Liang, Huawei Li, and Kwang-Ting Cheng.
\newblock Reliability evaluation and analysis of fpga-based neural network
  acceleration system.
\newblock {\em IEEE Transactions on Very Large Scale Integration (VLSI)
  Systems}, 29(3):472--484, 2021.

\bibitem{xue2022winograd}
Xinghua Xue, Haitong Huang, Cheng Liu, Ying Wang, Tao Luo, and Lei Zhang.
\newblock Winograd convolution: A perspective from fault tolerance.
\newblock {\em arXiv preprint arXiv:2202.08675}, 2022.

\bibitem{zhang2018analyzing}
Jeff~Jun Zhang, Tianyu Gu, Kanad Basu, and Siddharth Garg.
\newblock Analyzing and mitigating the impact of permanent faults on a systolic
  array based neural network accelerator.
\newblock In {\em 2018 IEEE 36th VLSI Test Symposium (VTS)}, pages 1--6. IEEE,
  2018.

\bibitem{Zhen2021MindFI}
Yang Zheng, Zhenye Feng, Zheng Hu, and Ke~Pei.
\newblock Mindfi: A fault injection tool for reliability assessment of
  mindspore applicacions.
\newblock In {\em 2021 IEEE International Symposium on Software Reliability
  Engineering Workshops (ISSREW)}, pages 235--238, 2021.

\end{thebibliography}

\end{document}